\documentclass{article}

\usepackage{PRIMEarxiv}

\usepackage[utf8]{inputenc} 
\usepackage[T1]{fontenc}    
\usepackage{hyperref}       
\usepackage{url}            
\usepackage{booktabs}       
\usepackage{amsmath, amssymb, amsfonts, amsthm}       
\usepackage{nicefrac}       
\usepackage{microtype}      
\usepackage{lipsum}
\usepackage{multirow}
\usepackage{fancyhdr}       
\usepackage{graphicx}       
\graphicspath{{media/}}     

\newtheorem{theorem}{Proposition}
\newtheorem{lemma}{Lemma}
\newtheorem{remark}{Remark}
\newtheorem{corollary}{Corollary}

\title{Learning Diffeomorphism for Image Registration with Time-Continuous Networks using Semigroup Regularization}

\author{
  Mohammadjavad Matinkia \\
  University of Alberta \\
  Edmonton, AB, Canada\\
  \texttt{matinkia@ualberta.ca} \\
  \And
  Nilanjan Ray \\
  University of Alberta \\
  Edmonton, AB, Canada\\
  \texttt{nray1@ualberta.ca} \\
}

\begin{document}
\maketitle

\begin{abstract}
\large
Diffeomorphic image registration (DIR) is a fundamental task in 3D medical image analysis that seeks topology-preserving deformations between image pairs. To ensure diffeomorphism, a common approach is to model the deformation field as the flow map solution of a differential equation, which is solved using efficient schemes such as scaling and squaring along with multiple smoothness regularization terms. In this paper, we propose a novel learning-based approach for diffeomorphic 3D image registration that models diffeomorphisms in a continuous-time framework using only a single regularization term, without requiring additional integration. We exploit the semigroup property—a fundamental characteristic of flow maps—as the sole form of regularization, ensuring temporally continuous diffeomorphic flows between image pairs. Leveraging this property, we prove that our formulation directly learns the flow map solution of an ODE, ensuring continuous inverse and cycle consistencies without explicit enforcement, while eliminating additional integration schemes and regularization terms. To achieve time-continuous diffeomorphisms, we employ time-embedded UNets, an architecture commonly used in diffusion models. Our results demonstrate that modeling diffeomorphism continuously in time improves registration performance. Experimental results on four public datasets demonstrate the superiority of our model over state-of-the-art diffeomorphic methods. Additionally, comparison to several recent non-diffeomorphic deformable image registration methods shows that our method achieves competitive Dice scores while significantly improving topology preservation. (\url{https://github.com/mattkia/SGDIR/})
\end{abstract}


\section{Introduction}
\label{sec:intro}
\large
Diffeomorphic image registration (DIR) is a computational technique used to align images into a common coordinate system using a deformation field, ensuring that the transformation is smooth and invertible \cite{survey1, survey2, synthmorph}. Unlike simpler forms of image registration that might allow folding or tearing, diffeomorphic methods preserve the topological properties of the image since the mapping between corresponding points is continuous, invertible, and differentiable. This property is crucial in medical imaging, where anatomical structures need to be aligned accurately without grid folding to facilitate precise analysis and diagnosis.

Addressing the problem of diffeomorphic image registration has led to the development of various approaches, broadly categorized into traditional, instance optimization-based, and learning-based methods. Traditional methods such as \textbf{SyN} \cite{syn}, \textbf{NiftyReg} \cite{niftyreg}, \textbf{B-Splines} \cite{bspliens}, and \textbf{LDDMM} \cite{lddmm} usually start with an assumption on the model functional form and try to maximize the similarity between the moving and the fixed images while constraining the model to satisfy some regularity properties. Even though such methods can usually generate fine-detailed registrations, they suffer from relatively long computation times. Recent advances in deep neural networks have yielded instance optimization-based approaches using neural networks which alleviate the restriction of having an assumption on the model functional and replace the registration model with a neural network with more expressive power \cite{idir, nodeo, dnvf, nih, homeomorphic, opt1}. Such models have demonstrated superior performance and lower computation times compared to traditional methods. On the other hand, learning-based approaches \cite{learning-based1, learning-based2, learning-based3, learning-based4, learnin-based5, learning-based7, symnet, voxelmorph} have gained traction due to their ability to learn implicit representations of the registration process from large datasets. These methods, still leveraging deep neural network architectures can offer potentially more robust solutions, often significantly reducing inference time compared to traditional and neural optimization-based methods.

Instead of directly modeling the deformation field, most of optimization- and learning-based methods \cite{dnvf, nih, symnet} model the velocity field controlling the dynamics of the deformation and use \textit{scaling and squaring} integration scheme to reconstruct the deformation \cite{ss, ss2}. Such methods typically regularize their model with various regularization terms to directly penalize the foldings in the deformation and impose smoothness on the vector field gradients and magnitudes. The advent of time-embedded architectures and time-embedded UNets \cite{diffusion1, diffusion2, stable-diffusion} has facilitated the incorporation of the notion of time into the neural networks. Models such as \textbf{DiffuseMorph} and \textbf{DiffuseReg} \cite{diffusemorph, diffusereg} directly employ the Diffusion Model architecture to directly and continuously learn the deformation.

It is also known that the inverse consistency of a registration model plays a critical point in the behavior of the registration method especially for the medical image registration task \cite{tmsc}. To this end, various methods have incorporated cycle and inverse consistencies either into the model's architecture or in their training procedure \cite{tmsc, cyclemorph, consistency1, consistency2}.

In this paper, we present a novel DIR method, called \textbf{SGDIR (Semi-Group DIR)}, which is capable of directly generating the continuous deformation without using scaling and squaring integration and multiple regularization terms. Our method is built upon a simple yet important property of flow maps, called the \textit{semigroup property} \cite{semigroup} which governs the composition rule of flows along time. Using time-embedded UNets our model is able to retrieve the deformation at any time step instantaneously, and hence provides a smooth transition, diffeomorphic at any time step, from the moving image to the fixed image, and vice versa. We will prove that our particular semigroup regularization ensures that the model learns a deformation field that is indeed the solution to an ODE, hence, implicitly satisfying the diffeomrophic property and cycle and inverse consistencies, without using any other regularization terms including penalizing negative Jacobian determinants of the deformation. 
Our contributions can be summarized as the followings:
\begin{itemize}
    \item We introduce a novel DIR method utilizing time-embedded UNets which is capable of producing diffeomophic deformation fields in a time continuum.
    \item We introduce a simple semigroup regularization as the only regularization term and we mathematically prove that our regularization term leads the model towards learning the solution of the deformation ODE, hence, implicitly satisfying the diffeomorphic property.
    \item We remove additional integration schemes such as scaling and squaring integration and regularization terms including penalization of negative Jacobian determinants of the deformation and imposition of smoothness of vector field gradient and magnitude.
    \item We provide extensive experiments on four 3D image registration datasets on brain MR scans, which demonstrate the superior accuracy of SGDIR compared to recent DIR methods. When compared with recent non-diffeomorphic deformable methods, SGDIR show significant improvement in topology preservation without sacrificing accuracy.
\end{itemize}

\section{Related Works}
\label{sec:related}
\large
\subsection{Pairwise Optimization-Based Methods}
Traditional methods find the deformation field between a pair of images by minimizing an energy functional while restricting the solution space by imposing some regularity conditions on the model. \textbf{SyN} \cite{lddmm} as a successful traditional model provides a greedy technique for a symmetric diffeomorphic deformation solution by minimizing the cross-correlation between the pair of images. \textbf{NiftyReg} \cite{niftyreg} as a powerful package for image registration considers different sets of linear and non-linear parametric models \cite{parametric} and minimizes loss functions such as Normalized Mutual Information (NMI) and Sum of Squared Differences (SSD) to maximize the alignment between the pairs of images. Traditional methods such as \textbf{Demons} and its variants \cite{demon1, demon2, demon3} are built upon optical flows. \textbf{LDDMM} \cite{lddmm} is a seminal work in flow-based registration that models the registration as the geodesic paths in the images space and is proper for large deformations while preserving the topology. \textbf{DARTEL} \cite{dartel} is also another flow-based approach which models the deformations using exponentiated Lie algebras and utilizes scaling and squaring integration to reconstruct the deformations.

Neural optimization based approaches model registration using neural networks. \textbf{IDIR} \cite{idir}, \textbf{DNVF} \cite{dnvf}, and \textbf{NePhi} \cite{nephi} are coordinate-based registration methods which utilize implicit neural representations \cite{siren}. IDIR directly learns the deformation and applies Jacobian regularization for better topology preservation and hyperelastic regularization for smooth deformation \cite{hyperelastic}, while DNVF implicitly learns the velocity field and using scaling and squaring integration, reconstructs the deformation. DNVF uses Jacobian regularization along with vector field gradient and magnitude regularization. \textbf{NODEO} \cite{nodeo} which also works with voxel coordinates, incorporates neural ODEs \cite{neuralODE} to learn the dynamics of the deformation by estimating the velocity field and finding deformation using an ODE solver.

\subsection{Learning-Based Methods}
Unlike optimization-based methods which perform pairwise optimization, learning-based methods try to learn the deformation across large training datasets. The most advantageous aspect of learning-based methods is their short inference time which is orders of magnitude shorter than that of optimization based methods. As a pioneering work, \textbf{VoxelMorph} \cite{voxelmorph} directly models the deformation using a UNet and by using localized Normalized Cross Correlation metric and smoothness regularization learns the deformation fields. \textbf{SYMNet} \cite{symnet} learns the velocity fields of forward and backward deformation and using scaling and squaring integration reconstructs the deformation. SYMNet uses Jacobian and smoothness regularizations along with minimizing the discrepancy between the forward and backward velocity fields, leading to a smoother deformation. \textbf{LapIRN} \cite{lapirn} and \textbf{PULPo} \cite{pulpo} employ multi-resolution deformations for better generalization and avoiding local minima. \textbf{cLapIRN} \cite{learning-based4} extends LapIRN and feeds the smoothness regularization factors to the layers of the network. \textbf{MS-ODE} \cite{MSODE} incorporates neural ODEs as a refinement stage by modeling the dynamics of the parameters of a registration model. \textbf{TransMorph} \cite{transmorph} introduces a general deformable image registration method by utilizing Vision Transformers, and further they customize their model for diffeomorphic registration, known as \textbf{TransMorph-diff}. \textbf{H-ViT} also uses the vision transformers along with cross-attention mechanism to produce more accurate deformations \cite{hvit}.

\textbf{DiffuseMorph} \cite{diffusemorph} utilizes denoising diffusion models to learn temporally continuous deformations. It also incorporates Jacobian and smoothness regularizations for imposing regularity on the solution. Instead of focusing on the images, \textbf{DiffuseReg} denoises the deformation field and follows the diffusion models strategy \cite{diffusereg}. However, using diffusion models could be challenging due to expensive training, maintaining visually pleasing outputs, and slow generation. We also use time-embedded UNets to learn the temporally continuous deformations. Unlike the diffusion model paradigm, we are able to output the deformation instantly at any timestep instead of building it from time step zero. Our simple design endows us with the ability to generate the deformation fields significantly faster.

\section{Preliminaries}
\label{sec:prem}
\large
\begin{figure}[!t]
    \centering
    \includegraphics[width=\linewidth]{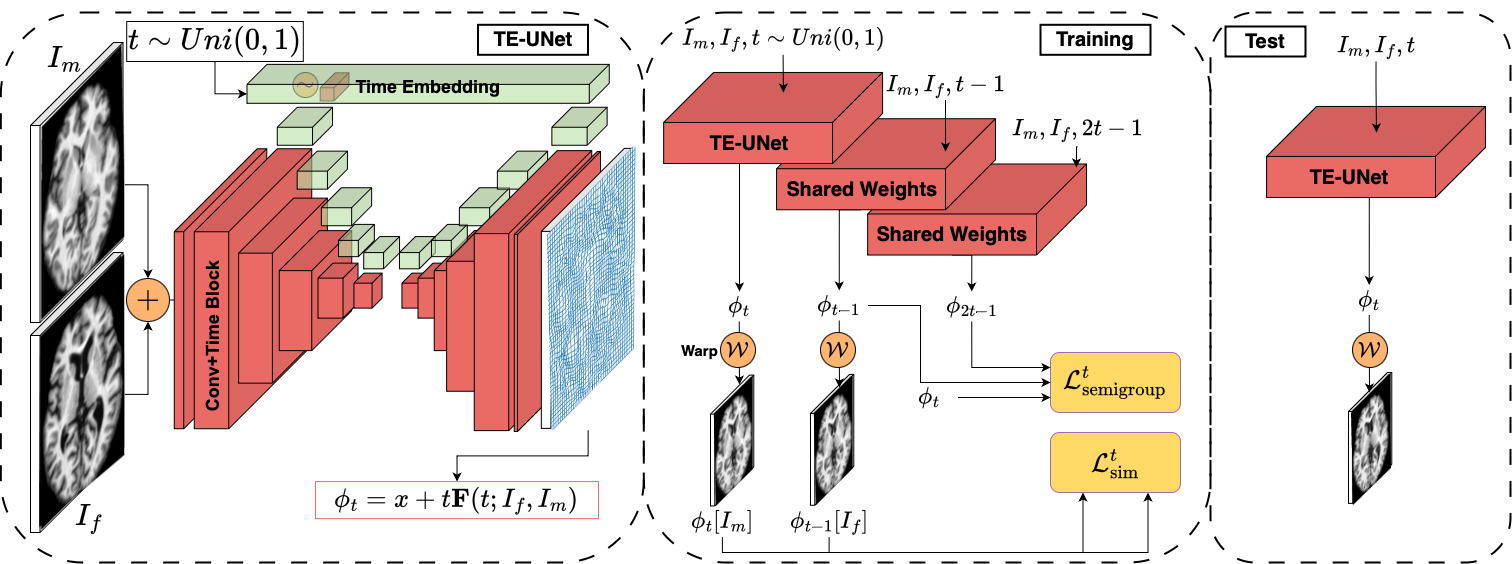}
    \caption{Schematic of SGDIR. As illustrated in the figure, there are no additional integration or extra loss functions to make the whole training procedure simple. During inference, we can query the model for any desired $t$, e.g., $t=1$ for the final deformation.}
    \label{fig:schematic}
\end{figure}

Deformable registration for 3D images seeks for a vector field $\phi:\mathbb{R}^3\longrightarrow\mathbb{R}^3$ that when applied to the moving image $I_m:\Omega\subset\mathbb{R}^3\longrightarrow\mathbb{R}$, deforms it smoothly towards the fixed image $I_f: \Omega\subset\mathbb{R}^3\longrightarrow\mathbb{R}$. The domain $\Omega\subset\mathbb{R}^3$ is usually a cubical region. In diffeomorphic image registration, we additionally require the deformation field to be an orientation-preserving diffeomorphism, which can be mathematically translated to $\phi$ having positive determinant of Jacobian $|J_\phi|$ at all points. The deformation field $\phi: \Omega \times \mathbb{R}\longrightarrow\mathbb{R}^3$ (we will use this definition going forward) can also be considered to be the time-continuous flow map solution of the following autonomous ordinary differential equation \cite{flow1, flow2, lddmm, dartel, symnet}:
\begin{gather}
\label{eq:ode}
    \left\{
    \begin{array}{ll}
        \frac{d\phi}{dt} = \mathbf{v}(\phi) & \\
        \phi_0(x) = x, & \\
    \end{array}
    \right.
\end{gather}
where $\phi_0(x)$ is short for $\phi(x, t=0)$ and $\mathbf{v}: \mathbb{R}^3\longrightarrow \mathbb{R}^3$ is a stationary velocity field that governs the dynamics of the flow. In the context of DIR, $\phi_1(x)$ is taken to be the deformation that warps $I_m$ to $I_f$. The utility of the ODE (\ref{eq:ode}) is that its solution is guaranteed to be a diffeomorphism. It is shown that considering 
\begin{gather}
    \label{eq:init}
    \phi_\frac{1}{2^N}(x) = x + \frac{\mathbf{v}}{2^N}
\end{gather}

for a relatively large $N$ (e.g., $N=8$) allows one to use the \textit{scaling and squaring} integration scheme to find $\phi_1(x)$ by iteratively applying the deformations \cite{ss, ss2}:
\begin{multline}
    \phi_{\frac{1}{2^{N-1}}} = \phi_{\frac{1}{2^N}} \circ \phi_{\frac{1}{2^N}}\Rightarrow\phi_{\frac{1}{2^{N-2}}} = \phi_{\frac{1}{2^{N-1}}} \circ \phi_{\frac{1}{2^{N-1}}}
    \Rightarrow\dots\Rightarrow\phi_1 = \phi_{1/2} \circ \phi_{1/2}.
\end{multline}

Conventional DIR methods, taking advantage of the above properties, assume the following general loss function to address the diffeomorphic registration problem.
\begin{gather}
\label{eq:loss_general}
    \mathcal{L} = \mathbb{E}_{(I_f, I_m) \sim \mathcal{D}}\left[\mathcal{L}_{\text{sim}}(I_f, \phi[I_m]) + \mathcal{L}_{\text{reg}}(\phi)\right],
\end{gather}
where $\mathcal{D}$ is the distribution of the training images, $\mathcal{L}_\text{sim}(.,.)$ is a measure of similarity between the fixed image and the warped moving image, and $\mathcal{L}_\text{reg}(\phi)$ contains restrictive regularization terms over the deformation $\phi$. Here, $\phi[I]$ means the warping of the image $I$ under the deformation field $\phi$. Starting with an estimation of the vector field $\mathbf{v}$ in Eq. \ref{eq:init}, the deformation $\phi$ (more precisely, $\phi_1$) is obtained by scaling and squaring integration, and the whole model is optimized to minimize the loss Eq. \ref{eq:loss_general}.

On the other hand, a necessary and sufficient condition of $\phi$ as the flow map solution of the ODE (\ref{eq:ode}) is that it satisfies the \textit{semigroup property}; i.e., for any time steps $t$ and $s$ it holds \cite{semigroup}:
\begin{gather}
   \phi(x,0)=x ~\text{and} ~\phi(\phi(x, s), t) = \phi(\phi(x, t), s) = \phi(x, s+t),
\end{gather}
or in short,
\begin{gather}
\label{eq:sg}
    \phi_0 = Id ~\text{and}~\phi_t\circ\phi_s = \phi_s\circ\phi_t = \phi_{t+s},
\end{gather}
where $Id$ denotes identity function. Intuitively, semigroup property implies that deforming an image $I$ up to a time $t+s$ is equivalent to first deforming the image up to time $t$ to obtain $I_t$, and then applying the deformation again up to time $s$ given that the initial condition has changed to $I_t$. Also, since $\phi_t\circ \phi_{-t}(x) = \phi_{-t}\circ\phi_t(x)  = \phi_0(x) = x$, we deduce that
\begin{gather}
\label{eq:inv}
    \phi^{-1}_t(x) = \phi_{-t}(x),
\end{gather}
for all values of $x$. Given that $\phi_t$ is a differentiable function, Eq. \ref{eq:inv} guaranties the bijectivity of $\phi_t(x)$, and hence, the diffeomorphic property of $\phi_t$. In the next section we'll show how we exploit the semigroup property of Eq. \ref{eq:sg} to circumvent the scaling and squaring integration and other regularization terms and perform the registration in a continuous time frame in contrast to former DIR methods which find the deformation at time $1$.

\section{Proposed Method}
\label{sec:method}
\large
Here we elaborate our method, dubbed as SGDIR. Let $I_f, I_m: \Omega\longrightarrow\mathbb{R}$ be the fixed and moving images, respectively, sampled from the image distribution $\mathcal{D}$. The domain $\Omega\subset\mathbb{R}^3$ is a $D\times H\times W$ cubical region. We are looking for a time-continuous deformation field $\phi: \Omega \times \mathbb{R}\longrightarrow\mathbb{R}^3$ to deform the moving image towards the fixed image, and vice versa. We assume that $\phi$ is the flow map solution of the ODE (\ref{eq:ode}), for some unknown stationary vector field $\mathbf{v}$, and therefore, it must satisfy the semigroup property of Eq. \ref{eq:sg}. If we warp $I_m$ up to time $t$ using the deformation $\phi_t$, and warp $I_f$ up to time $1-t$ using the inverse deformation $\phi^{-1}_{1-t}$, we must reach to a same point due to the continuity of the trajectory of $\phi$. Based on Eq. \ref{eq:inv}, we can write
\begin{gather}
\label{eq:cont_sim}
    \phi^{-1}_{1-t}[I_f] = \phi_{t-1}[I_f] = \phi_{t}[I_m].
\end{gather}
Eq. \ref{eq:cont_sim} allows us to define a time-continuous similarity loss. For measuring the similarity, we use the localized \textbf{N}ormalized \textbf{C}ross \textbf{C}orrelation (\textbf{NCC}) between the images \cite{voxelmorph, nodeo} (see Appendix \ref{app:ncc}).
In this work we use a local window of size $11$ for NCC. Using Eq. \ref{eq:cont_sim}, we introduce a time-continuous similarity loss:
\begin{align}
\label{eq:loss_sim}
    \mathcal{L}^{t}_\text{sim} &= -NCC(\phi_{t-1}[I_f], \phi_t[I_m]), && ~\forall t\in[0, 1].
\end{align}
We propose to learn the time-continuous deformation $\phi(x, t)$ using a time-embedded UNet as shown in Figure \ref{fig:schematic}. The Time-Embedded UNet (TE-UNet), frequently used in diffusion models, is capable of incorporating the notion of time and is suitable for learning flow maps. To achieve this goal, we model the deformation $\phi$ as
\begin{gather}
\label{eq:phi}
    \phi_t(x; \theta) = \phi(x, t;\theta) = x + t\mathbf{F}(x, t; I_f, I_m, \theta),
\end{gather}
for all $t\in[-1, 1]$, where $\mathbf{F}(x, t; I_f, I_m, \theta)$ is a TE-UNet with learnable parameters $\theta$ which also receives a pair of fixed and moving images. The reason behind the specific formulation of Eq. \ref{eq:phi} is that we can make sure that at $t=0$ we have $\phi_0(x) = x$, hence satisfying the identity initial condition of the ODE (\ref{eq:ode}). In order for $\phi$ to be a valid flow map, it needs to satisfy the semigroup property of Eq. \ref{eq:sg}. In practice, forcing the model to perfectly satisfy the semigroup property could be cumbersome and expensive, since we need to sample $t$ and $s$ independently and impose Eq. \ref{eq:loss_sim} twice, which requires more forward calls from the network. To tackle this issue we propose the following simplified semigroup regularization, 
\begin{gather}
\label{eq:reg}
    \mathcal{L}^{t}_\text{semigroup} = \|\phi_{2t-1} - \phi_t\circ\phi_{t-1}\|_2 + \|\phi_{2t-1} - \phi_{t-1}\circ\phi_t\|_2,
\end{gather}
for all $t\in[0, 1]$, and we show that it suffices to render $\phi$ a diffeomorphism at all time steps $t\in[-1, 1]$. In Eq. \ref{eq:reg}, $\|.\|_2$ denotes the $L_2$-norm, and forces the model to satisfy the semigroup property for $t$ and $s=t-1$. The second term in Eq. \ref{eq:reg} is added to impose symmetry and is essential to hold the semi-group property. Similar to all DIR methods we implement the composition operation by trilinear interpolation of grids over each other \cite{symnet, dnvf, lapirn, learning-based4}. From Eq. \ref{eq:reg}, we can say that a model that minimizes Eq. \ref{eq:reg}, will satisfy the following rules
\begin{gather}
\label{eq:composition}
    \phi_{2t-1} = \phi_t\circ\phi_{t-1},~\phi_{2t-1} = \phi_{t-1}\circ\phi_t.
\end{gather}

\begin{theorem}
\label{prop}
    The deformation $\phi(x, t)$ that satisfies the composition rule of Eq. \ref{eq:composition} and $\phi(x, 0) = x$ is an exponential map, or equivalently, it is a one-parameter diffeomorphism solving an autonomous ODE (\ref{eq:ode})(see \cite{biagi2018introduction}).
\end{theorem}
\begin{proof}
    See Appendix \ref{app:proof}.
\end{proof}

Finally, we can propose the following loss function which is in concordance with Eq. \ref{eq:loss_general}:
\begin{gather}
\label{eq:loss}
    \mathcal{L} = \mathbb{E}_{(I_f, I_m)\sim\mathcal{D}, t\sim Uni(0,1)}\left[\mathcal{L}^{t}_\text{sim} + \lambda \mathcal{L}^{t}_\text{semigroup}\right],
\end{gather}
where $Uni(0,1)$ is the uniform distribution on $[0,1]$ and $\mathcal{L}^{t}_\text{sim}$ and $\mathcal{L}^{t}_\text{semigroup}$ are defined as in Eq. \ref{eq:loss_sim} and Eq. \ref{eq:reg}, respectively. $\lambda$ is the regularization factor controlling the level of diffeomorphism. In the training phase, we randomly sample a pair of images to be registered along with a time step uniformly taken from the interval $[0, 1]$ to minimize Eq. \ref{eq:loss}. The trained model is capable of warping either $I_f$ or $I_m$ towards the other at any time step. Thus, the model is more versatile compared to the most previous works that provide the deformation in only one direction.

Moreover, note that since the implementation of $\phi_t$ as explained above has an inverse for any $t\in[-1,1]$ the network actually learns a \textbf{continuous group structure} of deformations over the interval $[-1, 1]$.

It is also worth mentioning that some prior works \cite{consistency1, consistency2, cyclemorph} enforce forward-backward deformation consistency as a regularization term by constructing a time-independent forward deformation only at the trajectory endpoint ($\phi_1$), while explicitly modeling the backward deformation ($\phi_1^{-1}$). In contrast, SGDIR allows evaluating the deformation field $\phi_t$ at arbitrary time steps, without separately modeling the inverse. Consistency between forward and backward deformations, as we prove, is obtained as a natural result of semigroup regularization, 

\begin{figure*}[!t]
    \centering
    \includegraphics[width=\textwidth]{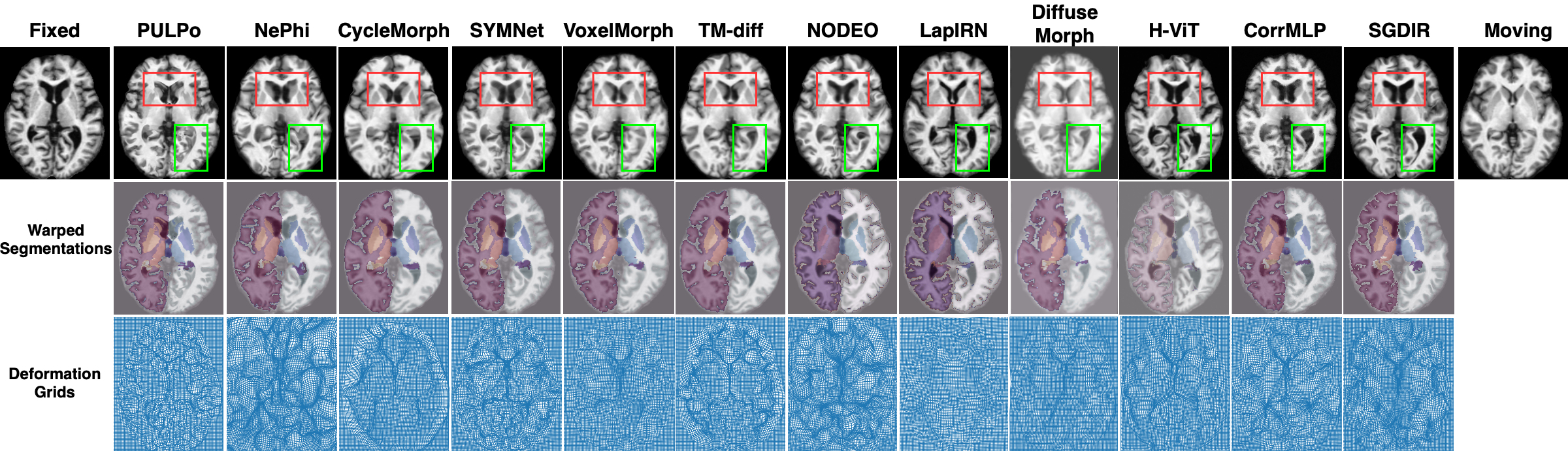}
    \caption{Qualitative comparison of SGDIR with other methods on the OASIS dataset. (visually looks better in color and zoomed in.)}
    \label{fig:comparisons}
\end{figure*}

\section{Experiments}
\label{sec:exp}
\large
\subsection{Datasets and Metrics}
\label{sec:datasets}

To evaluate SGDIR and compare its performance with other methods, we consider four publicly available datasets on 3D MR brain scans widely used in the literature, namely \textbf{OASIS} \cite{oasis}, \textbf{CANDI} \cite{candi}, \textbf{LPBA40} \cite{lpba}, and \textbf{IXI} \footnote{\url{https://brain-development.org/ixi-dataset/}}. To evaluate the performance of the model and for quantitative comparisons we use Dice score, 95th percentile Hausdorff Distance (HD95), and the percentage of negative Jacobian determinants. We provide the details of datasets, data preparations and evaluation metrics in Appendices \ref{app:data} and \ref{app:metrics}.

\subsection{Experimental Settings}
\label{sec:settings}
In the OASIS dataset, out of 416 subjects, we use 256, 60, 100 subjects for training, validation, and testing, respectively. By random sampling we generate 1000 pairs for training, 100 pairs for validation and 1000 pairs for test. 
In the CANDI dataset, out of 117 subjects, we use 80, 11, 26 subjects for training, validation, and testing, respectively.
We generate 400 pairs for training, 25 pairs for validation, and 300 pairs for test. For the LPBA40 dataset 190 pairs from 20 subjects are used as the training set, 10 pairs from 5 subjects are used as the validation set and 105 pairs from 15 subjects are used for the test set. Also, for the IXI dataset, out of 581 subjects, we use 450, 50, 81 subjects for training, validation, and testing, respectively, from which we create 1500, 100, 300 random pairs for training, validation, and test. 
We compare the performance of our model with a traditional optimization-based method, \textbf{SyN} \cite{syn}, as one of the most successful methods for diffeomorphic image registration, learning-based methods including \textbf{SYMNet} \cite{symnet}, \textbf{VoxelMorph} \cite{voxelmorph}, \textbf{LapIRN} \cite{lapirn}, \textbf{DiffuseMorph} \cite{diffusemorph}, \textbf{PULPo} \cite{pulpo}, \textbf{TransMorph-diff (TM-diff)} \cite{transmorph}, and optimization-based methods such as \textbf{IDIR} \cite{idir}, \textbf{NODEO} \cite{nodeo}, \textbf{NePhi} \cite{nephi}, \textbf{DNVF} and its variant \textbf{Cas-DNVF} \cite{dnvf}. We also compare our method with several non-diffeomorphic deformable methods including \textbf{H-ViT} \cite{hvit}, \textbf{CorrMLP} \cite{corrmlp}, \textbf{CycleMorph} \cite{cyclemorph}, and \textbf{TransMatch} \cite{transmatch}. For SyN, we used the implementation in the DIPY package \cite{dipy} and used cross-correlation as the loss function following the convention of the VoxelMorph. For other methods, except for DNVF/Cas-DNVF whose implementation was not available at the time of writing this paper, we used the codes available in their corresponding official repositories. The architecture of the UNet along with the time embedding modules is shown in Figure \ref{fig:schematic}. We have provided our full implementation setup in Appendix \ref{app:impl}. 

\subsection{Results}
Tables \ref{tab:dice}, \ref{tab:dice2}, \ref{tab:dice3}, \ref{tab:dice5} provide a comprehensive comparison of the performance of SGDIR with other aforementioned methods on OASIS, CANDI, LPBA40, and IXI, respectively. These tables provide the performance of models in terms of Dice score, HD95 metric, and the percentage of negative Jacobian determinants. All the values are expressed as their average over the test set along with their standard deviation. For DNVF and Cas-DNVF due to unavailability of the implementation the results are directly taken from the corresponding paper. We consider the percentage of negative Jacobian determinants smaller than $10^{-5}\%$ negligible and show it as $\approx 0$ in the tables. Also, Figure \ref{fig:comparisons} illustrates the qualitative comparison between our model and several other models on the OASIS dataset in terms of the warped image, warped segmentations, and the deformation grids. More visual results can be found in Appendix \ref{app:visuals}.

\begin{table}
    \centering
    \begin{minipage}[t]{0.49\linewidth}
        \caption{\small Comparison of the performance of SGDIR with other methods on the OASIS dataset. The values inside the parenthesis indicate the standard deviation.}
        \label{tab:dice}
        \centering
        \resizebox{\columnwidth}{!}{
        \begin{tabular}{lcccc}
          \toprule
          Category & Model & Dice ($\uparrow$) & HD95 ($\downarrow$) &   $|\text{det}(J_\phi) < 0|$ ($\downarrow$)\\
          \midrule
          Traditional & SyN & 73.14 (1.32) & 3.01 (0.28) & 0.0531\%   (0.0241)\\
          \midrule
          \multirow{3}{*}{Learning} & VoxelMorph & 75.93 (1.44) & 2.31   (0.24) & 0.0026\% (0.0017)\\
           & SYMNet & 79.39 (1.74) & 1.93 (0.37) & 0.0043\% (0.0017)\\
           & LapIRN & 78.80 (1.92) & 1.87 (0.23) & 0.0011\% (0.0003)\\
           & DiffuseMorph & 76.23 (2.23) & 2.34 (0.44) & 0.0551\%   (0.0131)\\
           & PULPo & 77.36 (1.48) & 2.05 (0.32) & 0.0137\% (0.0085) \\
           & CorrMLP & 80.78 (1.33) & 1.67 (0.33) & 0.1162\% (0.0310)\\
           & H-ViT & 81.14 (1.31) & 1.53 (0.26) & 0.31\% (0.012)\\
           & TransMatch & 80.02 (1.14) & 1.87 (0.34) & 0.42\% (0.093) \\
           & CycleMorph & 77.32 (2.34) & 1.97 (0.33) & 0.22\% (0.15) \\
           & TransMorph-diff & 79.84 (1.93) & 1.88 (0.23) & 0.0012   (0.0006) \\
          \midrule
          \multirow{4}{*}{Optimization} & IDIR & 72.31 (1.23) & 2.75   (0.34) & 0.0389\% (0.0105)\\
           & NODEO & 80.86 (1.14) & 1.55 (0.12) & 0.0105\% (0.0041)\\
           & NePhi & 77.89 (1.34) & 1.77 (0.13) & 0.0011\% (0.0003) \\
           & DNVF & 79.4 & - & 0.0015\%\\
           & Cas-DNVF & 81.5 & - & 0.0036\%\\
          \midrule
          Our Method & SGDIR & \textbf{82.23} (1.38) & \textbf{1.44}   (0.16) & \textbf{0.0009\%} (0.0007)\\
          \bottomrule
        \end{tabular}}
    \end{minipage}
    \hfill
    \begin{minipage}[t]{0.49\linewidth}
        \caption{\small Comparison of the performance of SGDIR with other methods on the CANDI dataset.}
  \label{tab:dice2}
  \centering
  \resizebox{\columnwidth}{!}{
  \begin{tabular}{lcccc}
    Category & Model & Dice ($\uparrow$) & HD95 ($\downarrow$) & $|\text{det}(J_\phi) < 0|$ ($\downarrow$)\\
    \midrule
    Traditional & SyN & 73.60 (1.15) & 2.13 (0.18) & 0.0011\% (0.0008) \\
    \midrule
    \multirow{3}{*}{Learning} & VoxelMorph & 75.63 (1.78) & 1.78 (0.24) &  0.0011\% (0.0009) \\
     & SYMNet & 75.89 (1.13) & 1.77 (0.24) & 0.0015\% (0.0008) \\
     & LapIRN & 76.23 (1.76) & 1.69 (0.21) & 0.0004\% (0.0002)\\
     & DiffuseMorph & 75.63 (2.01) & 1.83 (0.26) & 0.0041\% (0.0012) \\
     & PULPo & 77.97 (1.43) & 1.80 (0.17) & 0.0097\% (0.0011) \\
     & CorrMLP & 78.88 (1.92) & 1.78 (0.19) & 0.0018\% (0.0007) \\
     & H-ViT & 80.63 (1.10) & 1.66 (0.14) & 0.0476\% (0.0184) \\
     & CycleMorph & 77.31 (2.09) & 1.55 (0.13) & 0.0024\% (0.0011) \\
     & TransMorph-diff & 77.45 (1.39) & 1.51 (0.19) & 0.0021\% (0.0009) \\
    \midrule
    \multirow{3}{*}{Optimization} & IDIR & 74.42 (1.12) & 1.98 (0.11) & 0.0113\% (0.0081) \\
     & NODEO & 78.02 (1.52) & 1.76 (0.24) & $\approx$ 0\% \\
     & NePhi & 78.89 (1.44) & 1.73 (0.21) & 0.0004\% (0.0001) \\
     & DNVF & 78.6 & - & 0.0003\%\\
    \midrule
    Our Method & SGDIR & \textbf{80.73} (1.25) & \textbf{1.33} (0.20) & $\mathbf{\approx 0}$\% \\
    \bottomrule
  \end{tabular}}
    \end{minipage}
\end{table}

\begin{table}
    \centering
    \begin{minipage}[t]{0.49\linewidth}
        \caption{Comparison of the performance of SGDIR with other methods on the LPBA40 dataset.}
  \label{tab:dice3}
  \centering
  \resizebox{\columnwidth}{!}{
  \begin{tabular}{lcccc}
    Category & Model & Dice ($\uparrow$) & HD95 ($\downarrow$) & $|\text{det}(J_\phi) < 0|$ ($\downarrow$)\\
    \midrule
    Traditional & SyN & 66.82 (3.17) & 5.08 (0.34) & 0.0194\% (0.0072) \\
    \midrule
    \multirow{3}{*}{Learning} & VoxelMorph & 70.72 (2.33) & 4.02 (0.23) &  0.0023\% (0.0014) \\
     & SYMNet & 69.98 (1.37) & 4.06 (0.17) & 0.0018\% (0.0012) \\
     & LapIRN & 71.88 (1.36) & 3.98 (0.26) & 0.0011\% (0.0005)\\
     & DiffuseMorph & 68.73 (2.88) & 4.13 (0.32) & 0.0032\% (0.0011) \\
     & PULPo & 68.34 (1.88) & 4.36 (0.45) & 0.0052\% (0.0014) \\
     & CorrMLP & 73.12 (1.06) & 3.89 (0.21) & 0.0315\% (0.0021) \\
     & H-ViT & \textbf{74.87} (1.73) & \textbf{3.73} (0.19) & 0.0891\% (0.033) \\
     & TransMatch & 71.52 (1.66) & 4.01 (0.27) & 0.0981\% (0.0113) \\
     & CycleMorph & 71.03 (1.54) & 3.97 (0.32) & 0.0413\% (0.0087) \\
     & TransMorph-diff & 72.91 (1.46) & 3.83 (0.31) & 0.0093\% (0.0031) \\
    \midrule
    \multirow{2}{*}{Optimization} & IDIR & 67.93 (2.03) & 4.75 (0.43) & 0.0013\% (0.0008) \\
     & NODEO & 72.31 (1.12) & 3.96 (0.23) & 0.0009\% (0.0004) \\
     & NePhi & 69.39 (1.98) & 4.76 (0.74) & 0.0103\% (0.0045) \\
    \midrule
    Our Method & SGDIR & 74.23 (2.15) & 3.78 (0.13) & \textbf{0.0007\%} (0.0003) \\
    \bottomrule
  \end{tabular}}
    \end{minipage}
    \hfill
    \begin{minipage}[t]{0.49\linewidth}
        \caption{Comparison of the performance of SGDIR with other methods on the IXI dataset.}
  \label{tab:dice5}
  \centering
  \resizebox{\columnwidth}{!}{
  \begin{tabular}{lcccc}
    Category & Model & Dice ($\uparrow$) & HD95 ($\downarrow$) & $|\text{det}(J_\phi) < 0|$ ($\downarrow$)\\
    \midrule
    Traditional & SyN & 65.14 (2.09) & 5.12 (0.66) & 0.0201\% (0.0096) \\
    \midrule
    \multirow{3}{*}{Learning} & VoxelMorph & 70.76 (2.32) & 4.62 (0.31) &  0.0575\% (0.0043) \\
     & SYMNet & 71.21 (1.73) & 4.46 (0.27) & 0.0084\% (0.0033) \\
     & LapIRN & 72.71 (1.94) & 4.24 (0.37) & 0.0043\% (0.0012) \\
     & DiffuseMorph & 69.11 (2.14) & 4.91 (0.76) &  0.0175\% (0.0053) \\
     & PULPo & 69.24 (2.31) & 4.83 (0.55) &  0.0126\% (0.0044) \\
     & CorrMLP & 75.44 (1.32) & 3.76 (0.31) & 0.0192\% (0.0064) \\
     & H-ViT & 76.02 (1.45) & 3.52 (0.24) & 0.23\% (0.07) \\
     & TransMatch & 74.19 (1.24) & 3.89 (0.66) & 0.87\% (0.026) \\
     & CycleMorph & 70.39 (1.43) & 4.73 (0.54) & 0.47\% (0.013) \\
     & TransMorph-diff & 72.30 (2.01) & 4.08 (0.43) & 0.0049\% (0.0004)\\
    \midrule
    \multirow{2}{*}{Optimization} & IDIR & 68.25 (2.28) & 5.01 (0.68) &  0.0105\% (0.0033) \\
     & NODEO & 76.81 (1.09) & 3.11 (0.14) & 0.0013\% (0.0008) \\
     & NePhi & 73.81 (1.09) & 4.14 (0.38) & 0.0093\% (0.0027) \\
    \midrule
    Our Method & SGDIR & \textbf{77.98} (1.08) & \textbf{3.02} (0.12) & \textbf{0.0012\%} (0.0009) \\
    \bottomrule
  \end{tabular}}
    \end{minipage}
\end{table}

Tables \ref{tab:dice}, \ref{tab:dice2}, \ref{tab:dice3}, \ref{tab:dice5} indicate that SGDIR outperforms competent diffeomorphic methods both in terms of Dice score, and $|det(J_\phi)<0|$. SGDIR demonstrates 1.7\%, 2.3\%, 1.8\%, and 1.5\% improvements on Dice score with respect to the best evaluated diffeomorphic method over the OASIS, CANDI, LPBA40, and IXI datasets respectively. Among the non-diffeomorphic methods, H-ViT demonstrates the highest performance. The experiments demonstrate the SGDIR achieves competitive results against H-ViT in terms of Dice score while significantly improving the smoothness of deformation grid through $|det(J_\phi)<0|$. This implies that the imposition of the semigroup property as a regularization in the time continuum of the deformation field can in fact lead to a precise and smooth deformation.

Among the compared methods, DiffuseMorph is also capable of generating a temporal deformation flow. Therefore, we have provided a visual comparison of warped images within 7 time steps for the OASIS dataset in Figure \ref{fig:temporal}. In this example we have used the subject id 1 and 2 of the OASIS dataset as the fixed image and the moving image respectively. More visual results are provided in Appendix \ref{app:visuals}. Also, additional results on another dataset with a different modality is available in Appendix \ref{app:dirlab}.

\begin{figure*}
    \centering
    \includegraphics[width=\textwidth]{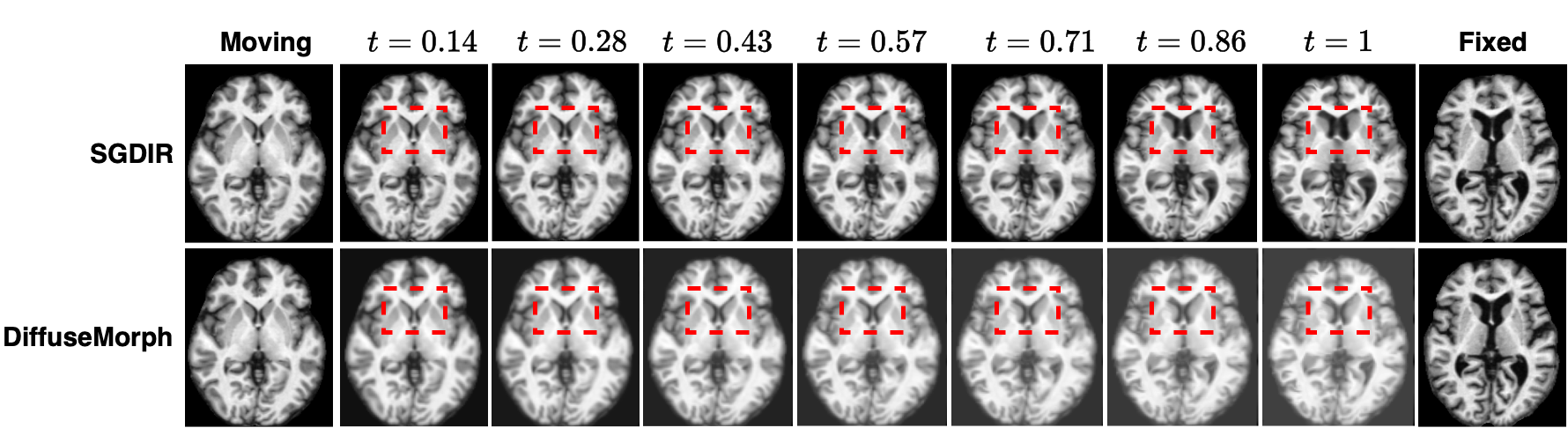}
    \caption{The visual comparison of the continuous deformations of DiffuseMorph and SGDIR for OASIS dataset.}
    \label{fig:temporal}
\end{figure*}

\subsection{Ablation Study}
Here we investigate the effect of some hyperparameters and design choices along with some consequent properties of SGDIR. More precisely, we investigate the effect of the semigroup regularization factor $\lambda$, the effect of continuous time scale against discrete time, and the topology-preserving property throughout the time interval. Since the behavior of the model is the same for all datasets, we provide the analysis only for the OASIS dataset.

\subsubsection{Effect of Regularization Factor}
\label{sec:lambda}

Eq. \ref{eq:reg} is responsible for imposing the semigroup property to the model and consequently, make the model produce topology-preserving non-folding deformations. Table \ref{tab:lambda} provides the effect of changing the factor $\lambda$ for the semigroup regularization on the Dice score and $|det(J_\phi)<0|$. This table indicates that for a choice of a very large $\lambda$ we could achieve zero-folding deformations with a decrease in the dice score. Decreasing $\lambda$ gives more freedom to the model to increase the similarity of the deformed image and the fixed image in exchange for a larger folding percentage. An interesting property exhibited here is that drastically decreasing $\lambda$ (e.g., $\lambda=1$) is not helpful for achieving the highest Dice score. This indicates that the semigroup property is also helpful in narrowing down the trajectory space of the deformations, helping the model to converge to the more optimal solution.

\begin{table}
  \caption{Analysis of the effect of the semigroup regularization factor on the Dice score and the percentage of negative Jacobian determinants over the OASIS dataset.}
  \label{tab:lambda}
  \centering
  \resizebox{0.4\textwidth}{!}{
  \begin{tabular}{lcc}
    \toprule
    $\lambda$ & \small{Dice} $(\uparrow)$ & $|\text{det}(J_\phi) < 0|$ $(\downarrow)$ \\
    \midrule
    $10^7$ & 76.20 & \textbf{0.0000\%} \\
    $10^5$ & 82.11 & 0.0011\% \\
    $10^4$ & \textbf{83.45} & 3.0691\% \\
    $10^3$ & 82.66 & 6.0955\% \\
    $10^2$ & 80.01 & 6.7403\% \\
    $10$ & 80.23 & 7.0794\% \\
    $1$ & 79.80 & 7.8612\% \\
    \bottomrule
  \end{tabular}}
\end{table}

\subsubsection{Diffeomorphism Through Time}
Applying the semigroup regularization ensures that the deformation is a diffeomorphism not only at the endpoints, but also at all times $t\in[0, 1]$. Figure \ref{fig:jdets} illustrates how $|det(J_\phi) < 0|$ varies across time. The negative time interval is used for reverse warping; i.e., warping the fixed image towards the moving image. The upper figure illustrates the effect of different values of $\lambda$ on $|det(J_\phi)<0|$ throughout the time interval, and one can see that a proper factor for the semigroup regularization could lead to an almost perfect diffeomorphism in each time step. The bottom figure, focuses on the performance of the best model. The purple plot and the right y-axis indicate the Dice score of intermediate warped images throughout time and the black plot and the left y-axis indicate the $|det(J_\phi)<0|$. The shaded areas for both plots show the variance of the corresponding metric. We observe that the semigroup regularization successfully preserves the topology at all time steps, while keeping the Dice score high during the transition. 

\begin{figure}[!ht]
    \centering
    \includegraphics[width=\textwidth]{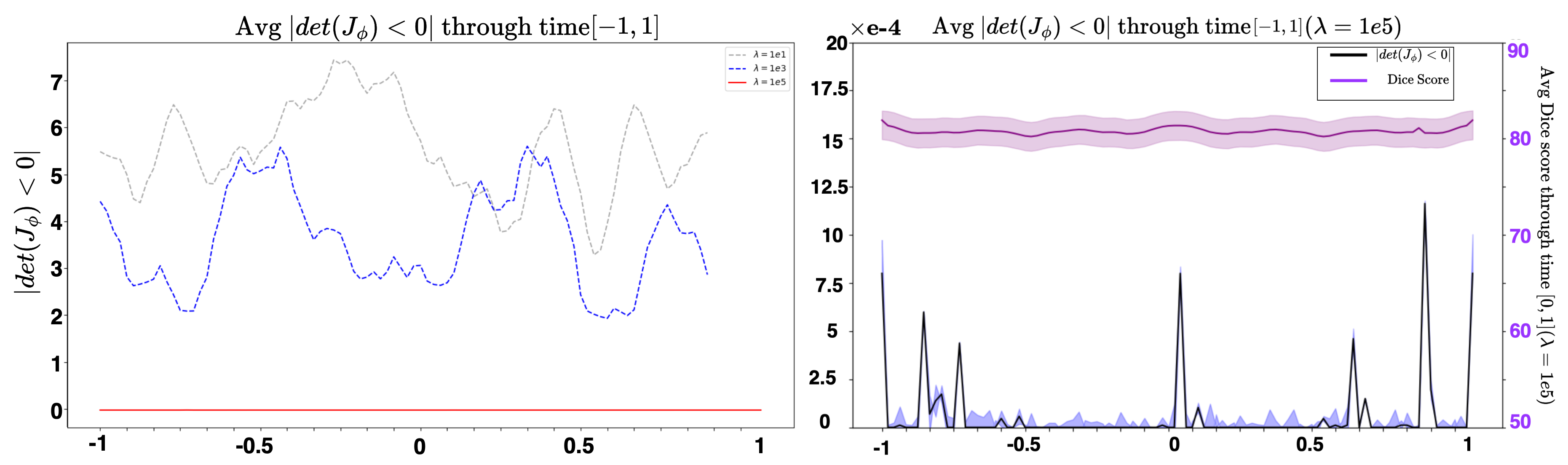}
    \caption{The percentage of negative Jacobian determinants through time. The upper figure compares the effect of $\lambda$ on the $|det(J_\phi)<0|$ and the bottom s figure focuses on the best configuration model. The shaded area in the bottom figure indicates the variance across the test set.}
    \label{fig:jdets}
\end{figure}

\subsubsection{Effect of Discrete Time Steps}
SGDIR can compute the deformation at an arbitrary time step using a continuous sampling from the time interval $[0,1]$. Here, we investigate the effect of using discrete time steps for training the model. The number of time steps indicates the number of intermediate time steps used during training. For instance, 0 time steps means that only the endpoint are used for registration (at time 1), and 8 time steps means that 8 equidistant intermediate time steps are used during training. $\infty$ is the continuous time sampling as done in our method. At the evaluation phase, we query the model at times 1 and -1 to warp the moving image towards the fixed image, and vice versa. Table \ref{tab:discrete} demonstrates that increasing the number of time samples in the interval $[0, 1]$ improves the performance of the model both in terms of Dice score and $|det(J_\phi)<0|$, while the best performance is achieved by continuous sampling. Figure \ref{fig:timsteps} in Appendix \ref{app:visuals} provides an illustrative inspection of this effect as well.

\begin{table}[!ht]
  \caption{Effect of the number time steps on the Dice score and the percentage of negative Jacobian determinants. }
  \label{tab:discrete}
  \centering
  \resizebox{0.4\textwidth}{!}{
  \begin{tabular}{lcc}
    \toprule
    \# Time Steps & \small{Dice} & $|\text{det}(J_\phi) < 0|$ \\
    \midrule
    $0$ & 73.12 & 13.121\% \\
    $1$ & 73.68 & 0.152\% \\
    $2$ & 74.23 & 0.065\% \\
    $4$ & 75.03 & 0.008\% \\
    $8$ & 75.45 & 0.002\% \\
    $\infty$ & \textbf{82.23} & \textbf{0.0009\%} \\
    \bottomrule
  \end{tabular}}
\end{table}

\begin{table}[!ht]
  \caption{Inference time comparison between SGDIR and other methods.}
  \label{tab:inftime}
  \centering
  \resizebox{0.4\textwidth}{!}{
  \begin{tabular}{ll}
    \toprule
    Model & Inference Time (s) $(\downarrow)$ \\
    \midrule
    SYMNet & 0.41 \\
    VoxelMorph & 0.41 \\
    TM-diff & 0.39 \\
    PULPo & 0.38 \\
    DiffuseMorph & 5.77 \\
    H-ViT & 0.31 \\
    \midrule
    NODEO & 58.37 \\
    IDIR & 149.24 \\
    NePhi & 171.86 \\
    \midrule
    SGDIR & \textbf{0.23} \\
    \bottomrule
  \end{tabular}}
\end{table}

\subsection{Computational Analysis}
\label{app:inftime}

SGDIR is comprised of a rather simple architecture and avoids scaling and squaring integration. Here we have compared the inference time of SGDIR with other methods to evaluate how removing the integration can help the method to have faster inference time. The inference times are calculated for a single pair of brain MR scans. Table \ref{tab:inftime} compares the inference times. We note that among the learning-based diffeomorphic methods, SGDIR is 1.8 times as fast as SYMNet which employs scaling and squaring during the inference. DiffuseMorph requires us to take the iterative sampling steps of diffusion models and consequently, has much higher inference time. Due to existence of computationally expensive self- and cross-attentions, H-ViT demonstrates inferior inference time comparing to SGDIR. Finally, instance optimization-based methods which need to optimize the deformation grid for separate pairs, demonstrate orders of magnitude higher inference times, which is expected from such methods. More details on the number of parameters of SGDIR and its memory use can be found in Appendix \ref{app:impl} and in Table \ref{tab:comp}.

\section{Conclusion}
\label{sec:conc}
\large
In this paper, we proposed SGDIR, a novel learning-based model for diffeomorphic image registration with a time-embedded UNet architecture and using a semigroup-based regularization. The time-embedded network allows us to compute deformation and consequently warp images at any time point, and the semigroup regularization enables the model to generate diffeomorphic deformations and its inverse. The resulting model is capable of generating smooth warpings from the moving image to the fixed image and vice versa without requiring any integration scheme or using further regularization terms on the smoothness and inverse consistency. We evaluated our model on four benchmark datasets on 3D medical image registration and demonstrated its superiority in various metrics compared to the state-of-the-art methods. In a broader perspective, the proposed framework provides a systematic method for addressing problems which can be modeled as an autonomous ODE.

\bibliographystyle{unsrt}  
\bibliography{main}

\begin{thebibliography}{10}

\bibitem{survey1}
Aristeidis Sotiras, Christos Davatzikos, and Nikos Paragios.
\newblock Deformable medical image registration: A survey.
\newblock {\em IEEE transactions on medical imaging}, 32(7):1153--1190, 2013.

\bibitem{survey2}
Nicholas~J Tustison, Brian~B Avants, and James~C Gee.
\newblock Learning image-based spatial transformations via convolutional neural networks: A review.
\newblock {\em Magnetic resonance imaging}, 64:142--153, 2019.

\bibitem{synthmorph}
Malte Hoffmann, Benjamin Billot, Douglas~N Greve, Juan~Eugenio Iglesias, Bruce Fischl, and Adrian~V Dalca.
\newblock Synthmorph: learning contrast-invariant registration without acquired images.
\newblock {\em IEEE transactions on medical imaging}, 41(3):543--558, 2021.

\bibitem{syn}
Brian~B Avants, Charles~L Epstein, Murray Grossman, and James~C Gee.
\newblock Symmetric diffeomorphic image registration with cross-correlation: evaluating automated labeling of elderly and neurodegenerative brain.
\newblock {\em Medical image analysis}, 12(1):26--41, 2008.

\bibitem{niftyreg}
Marc Modat, Gerard~R Ridgway, Zeike~A Taylor, Manja Lehmann, Josephine Barnes, David~J Hawkes, Nick~C Fox, and S{\'e}bastien Ourselin.
\newblock Fast free-form deformation using graphics processing units.
\newblock {\em Computer methods and programs in biomedicine}, 98(3):278--284, 2010.

\bibitem{bspliens}
Daniel Rueckert, Luke~I Sonoda, Carmel Hayes, Derek~LG Hill, Martin~O Leach, and David~J Hawkes.
\newblock Nonrigid registration using free-form deformations: application to breast mr images.
\newblock {\em IEEE transactions on medical imaging}, 18(8):712--721, 1999.

\bibitem{lddmm}
M~Faisal Beg, Michael~I Miller, Alain Trouv{\'e}, and Laurent Younes.
\newblock Computing large deformation metric mappings via geodesic flows of diffeomorphisms.
\newblock {\em International journal of computer vision}, 61:139--157, 2005.

\bibitem{idir}
Jelmer~M Wolterink, Jesse~C Zwienenberg, and Christoph Brune.
\newblock Implicit neural representations for deformable image registration.
\newblock In {\em International Conference on Medical Imaging with Deep Learning}, pages 1349--1359. PMLR, 2022.

\bibitem{nodeo}
Yifan Wu, Tom~Z Jiahao, Jiancong Wang, Paul~A Yushkevich, M~Ani Hsieh, and James~C Gee.
\newblock Nodeo: A neural ordinary differential equation based optimization framework for deformable image registration.
\newblock In {\em Proceedings of the IEEE/CVF conference on computer vision and pattern recognition}, pages 20804--20813, 2022.

\bibitem{dnvf}
Kun Han, Shanlin Sun, Xiangyi Yan, Chenyu You, Hao Tang, Junayed Naushad, Haoyu Ma, Deying Kong, and Xiaohui Xie.
\newblock Diffeomorphic image registration with neural velocity field.
\newblock In {\em Proceedings of the IEEE/CVF Winter Conference on Applications of Computer Vision}, pages 1869--1879, 2023.

\bibitem{nih}
Shanlin Sun, Kun Han, Hao Tang, Deying Kong, Junayed Naushad, Xiangyi Yan, and Xiaohui Xie.
\newblock Medical image registration via neural fields.
\newblock {\em arXiv preprint arXiv:2206.03111}, 2022.

\bibitem{homeomorphic}
Jing Zou, No{\'e}mie Debroux, Lihao Liu, Jing Qin, Carola-Bibiane Sch{\"o}nlieb, and Angelica~I Aviles-Rivero.
\newblock Homeomorphic image registration via conformal-invariant hyperelastic regularisation.
\newblock {\em arXiv preprint arXiv:2303.08113}, 2023.

\bibitem{opt1}
Zhengyang Shen, Xu~Han, Zhenlin Xu, and Marc Niethammer.
\newblock Networks for joint affine and non-parametric image registration.
\newblock In {\em Proceedings of the IEEE/CVF Conference on Computer Vision and Pattern Recognition}, pages 4224--4233, 2019.

\bibitem{learning-based1}
Guha Balakrishnan, Amy Zhao, Mert~R Sabuncu, John Guttag, and Adrian~V Dalca.
\newblock An unsupervised learning model for deformable medical image registration.
\newblock In {\em Proceedings of the IEEE conference on computer vision and pattern recognition}, pages 9252--9260, 2018.

\bibitem{learning-based2}
Adrian~V Dalca, Guha Balakrishnan, John Guttag, and Mert~R Sabuncu.
\newblock Unsupervised learning for fast probabilistic diffeomorphic registration.
\newblock In {\em Medical Image Computing and Computer Assisted Intervention--MICCAI 2018: 21st International Conference, Granada, Spain, September 16-20, 2018, Proceedings, Part I}, pages 729--738. Springer, 2018.

\bibitem{learning-based3}
Ameneh Sheikhjafari, Michelle Noga, Kumaradevan Punithakumar, and Nilanjan Ray.
\newblock Unsupervised deformable image registration with fully connected generative neural network.
\newblock In {\em Medical imaging with deep learning}, 2022.

\bibitem{learning-based4}
Tony~CW Mok and Albert~CS Chung.
\newblock Conditional deformable image registration with convolutional neural network.
\newblock In {\em Medical Image Computing and Computer Assisted Intervention--MICCAI 2021: 24th International Conference, Strasbourg, France, September 27--October 1, 2021, Proceedings, Part IV 24}, pages 35--45. Springer, 2021.

\bibitem{learnin-based5}
Yungeng Zhang, Yuru Pei, and Hongbin Zha.
\newblock Learning dual transformer network for diffeomorphic registration.
\newblock In {\em Medical Image Computing and Computer Assisted Intervention--MICCAI 2021: 24th International Conference, Strasbourg, France, September 27--October 1, 2021, Proceedings, Part IV 24}, pages 129--138. Springer, 2021.

\bibitem{learning-based7}
Adrian~V Dalca, Guha Balakrishnan, John Guttag, and Mert~R Sabuncu.
\newblock Unsupervised learning of probabilistic diffeomorphic registration for images and surfaces.
\newblock {\em Medical image analysis}, 57:226--236, 2019.

\bibitem{symnet}
Tony~CW Mok and Albert Chung.
\newblock Fast symmetric diffeomorphic image registration with convolutional neural networks.
\newblock In {\em Proceedings of the IEEE/CVF conference on computer vision and pattern recognition}, pages 4644--4653, 2020.

\bibitem{voxelmorph}
Guha Balakrishnan, Amy Zhao, Mert~R Sabuncu, John Guttag, and Adrian~V Dalca.
\newblock Voxelmorph: a learning framework for deformable medical image registration.
\newblock {\em IEEE transactions on medical imaging}, 38(8):1788--1800, 2019.

\bibitem{ss}
Vincent Arsigny, Olivier Commowick, Xavier Pennec, and Nicholas Ayache.
\newblock A log-euclidean framework for statistics on diffeomorphisms.
\newblock In {\em Medical Image Computing and Computer-Assisted Intervention--MICCAI 2006: 9th International Conference, Copenhagen, Denmark, October 1-6, 2006. Proceedings, Part I 9}, pages 924--931. Springer, 2006.

\bibitem{ss2}
Monica Hernandez, Matias~N Bossa, and Salvador Olmos.
\newblock Registration of anatomical images using geodesic paths of diffeomorphisms parameterized with stationary vector fields.
\newblock In {\em 2007 IEEE 11th International Conference on Computer Vision}, pages 1--8. IEEE, 2007.

\bibitem{diffusion1}
Jascha Sohl-Dickstein, Eric Weiss, Niru Maheswaranathan, and Surya Ganguli.
\newblock Deep unsupervised learning using nonequilibrium thermodynamics.
\newblock In {\em International conference on machine learning}, pages 2256--2265. PMLR, 2015.

\bibitem{diffusion2}
Jonathan Ho, Ajay Jain, and Pieter Abbeel.
\newblock Denoising diffusion probabilistic models.
\newblock {\em Advances in neural information processing systems}, 33:6840--6851, 2020.

\bibitem{stable-diffusion}
Robin Rombach, Andreas Blattmann, Dominik Lorenz, Patrick Esser, and Bj{\"o}rn Ommer.
\newblock High-resolution image synthesis with latent diffusion models.
\newblock In {\em Proceedings of the IEEE/CVF conference on computer vision and pattern recognition}, pages 10684--10695, 2022.

\bibitem{diffusemorph}
Boah Kim, Inhwa Han, and Jong~Chul Ye.
\newblock Diffusemorph: Unsupervised deformable image registration using diffusion model.
\newblock In {\em European conference on computer vision}, pages 347--364. Springer, 2022.

\bibitem{diffusereg}
Yongtai Zhuo and Yiqing Shen.
\newblock Diffusereg: Denoising diffusion model for obtaining deformation fields in unsupervised deformable image registration.
\newblock In {\em International Conference on Medical Image Computing and Computer-Assisted Intervention}, pages 597--607. Springer, 2024.

\bibitem{tmsc}
Bin Duan, Ming Zhong, and Yan Yan.
\newblock Towards saner deep image registration.
\newblock In {\em Proceedings of the IEEE/CVF International Conference on Computer Vision}, pages 12459--12468, 2023.

\bibitem{cyclemorph}
Boah Kim, Dong~Hwan Kim, Seong~Ho Park, Jieun Kim, June-Goo Lee, and Jong~Chul Ye.
\newblock Cyclemorph: cycle consistent unsupervised deformable image registration.
\newblock {\em Medical image analysis}, 71:102036, 2021.

\bibitem{consistency1}
Gary~E Christensen, Joo~Hyun Song, Wei Lu, Issam El~Naqa, and Daniel~A Low.
\newblock Tracking lung tissue motion and expansion/compression with inverse consistent image registration and spirometry.
\newblock {\em Medical physics}, 34(6Part1):2155--2163, 2007.

\bibitem{consistency2}
Gary~E Christensen and Hans~J Johnson.
\newblock Consistent image registration.
\newblock {\em IEEE transactions on medical imaging}, 20(7):568--582, 2001.

\bibitem{semigroup}
Stefano Biagi and Andrea Bonfiglioli.
\newblock {\em An Introduction to the Geometrical Analysis of Vector Fields: with Applications to Maximum Principles and Lie Groups}.
\newblock World Scientific, 2019.

\bibitem{parametric}
Dinggang Shen and Christos Davatzikos.
\newblock Hammer: hierarchical attribute matching mechanism for elastic registration.
\newblock {\em IEEE transactions on medical imaging}, 21(11):1421--1439, 2002.

\bibitem{demon1}
J-P Thirion.
\newblock Non-rigid matching using demons.
\newblock In {\em Proceedings CVPR IEEE Computer Society Conference on Computer Vision and Pattern Recognition}, pages 245--251. IEEE, 1996.

\bibitem{demon2}
Tom Vercauteren, Xavier Pennec, Aymeric Perchant, and Nicholas Ayache.
\newblock Symmetric log-domain diffeomorphic registration: A demons-based approach.
\newblock In {\em International conference on medical image computing and computer-assisted intervention}, pages 754--761. Springer, 2008.

\bibitem{demon3}
Tom Vercauteren, Xavier Pennec, Aymeric Perchant, and Nicholas Ayache.
\newblock Diffeomorphic demons: Efficient non-parametric image registration.
\newblock {\em NeuroImage}, 45(1):S61--S72, 2009.

\bibitem{dartel}
John Ashburner.
\newblock A fast diffeomorphic image registration algorithm.
\newblock {\em Neuroimage}, 38(1):95--113, 2007.

\bibitem{nephi}
Lin Tian, Soumyadip Sengupta, Hastings Greer, Ra{\'u}l San~Jos{\'e} Est{\'e}par, and Marc Niethammer.
\newblock Nephi: Neural deformation fields for approximately diffeomorphic medical image registration.
\newblock {\em arXiv preprint arXiv:2309.07322}, 2023.

\bibitem{siren}
Vincent Sitzmann, Julien Martel, Alexander Bergman, David Lindell, and Gordon Wetzstein.
\newblock Implicit neural representations with periodic activation functions.
\newblock {\em Advances in neural information processing systems}, 33:7462--7473, 2020.

\bibitem{hyperelastic}
Martin Burger, Jan Modersitzki, and Lars Ruthotto.
\newblock A hyperelastic regularization energy for image registration.
\newblock {\em SIAM Journal on Scientific Computing}, 35(1):B132--B148, 2013.

\bibitem{neuralODE}
Ricky~TQ Chen, Yulia Rubanova, Jesse Bettencourt, and David~K Duvenaud.
\newblock Neural ordinary differential equations.
\newblock {\em Advances in neural information processing systems}, 31, 2018.

\bibitem{lapirn}
Tony~CW Mok and Albert~CS Chung.
\newblock Large deformation diffeomorphic image registration with laplacian pyramid networks.
\newblock In {\em Medical Image Computing and Computer Assisted Intervention--MICCAI 2020: 23rd International Conference, Lima, Peru, October 4--8, 2020, Proceedings, Part III 23}, pages 211--221. Springer, 2020.

\bibitem{pulpo}
Leonard Siegert, Paul Fischer, Mattias~P Heinrich, and Christian~F Baumgartner.
\newblock Pulpo: Probabilistic unsupervised laplacian pyramid registration.
\newblock In {\em International Conference on Medical Image Computing and Computer-Assisted Intervention}, pages 717--727. Springer, 2024.

\bibitem{MSODE}
Junshen Xu, Eric~Z Chen, Xiao Chen, Terrence Chen, and Shanhui Sun.
\newblock Multi-scale neural odes for 3d medical image registration.
\newblock In {\em Medical Image Computing and Computer Assisted Intervention--MICCAI 2021: 24th International Conference, Strasbourg, France, September 27--October 1, 2021, Proceedings, Part IV 24}, pages 213--223. Springer, 2021.

\bibitem{transmorph}
Junyu Chen, Eric~C Frey, Yufan He, William~P Segars, Ye~Li, and Yong Du.
\newblock Transmorph: Transformer for unsupervised medical image registration.
\newblock {\em Medical image analysis}, 82:102615, 2022.

\bibitem{hvit}
Morteza Ghahremani, Mohammad Khateri, Bailiang Jian, Benedikt Wiestler, Ehsan Adeli, and Christian Wachinger.
\newblock H-vit: A hierarchical vision transformer for deformable image registration.
\newblock In {\em Proceedings of the IEEE/CVF Conference on Computer Vision and Pattern Recognition}, pages 11513--11523, 2024.

\bibitem{flow1}
Paul Dupuis, Ulf Grenander, and Michael~I Miller.
\newblock Variational problems on flows of diffeomorphisms for image matching.
\newblock {\em Quarterly of applied mathematics}, pages 587--600, 1998.

\bibitem{flow2}
Michael~I Miller, Alain Trouv{\'e}, and Laurent Younes.
\newblock On the metrics and euler-lagrange equations of computational anatomy.
\newblock {\em Annual review of biomedical engineering}, 4(1):375--405, 2002.

\bibitem{biagi2018introduction}
S.~Biagi and A.~Bonfiglioli.
\newblock {\em An Introduction to the Geometrical Analysis of Vector Fields: With Applications to Maximum Principles and Lie Groups}.
\newblock World Scientific, 2018.

\bibitem{oasis}
Daniel~S Marcus, Tracy~H Wang, Jamie Parker, John~G Csernansky, John~C Morris, and Randy~L Buckner.
\newblock Open access series of imaging studies (oasis): cross-sectional mri data in young, middle aged, nondemented, and demented older adults.
\newblock {\em Journal of cognitive neuroscience}, 19(9):1498--1507, 2007.

\bibitem{candi}
David~N Kennedy, Christian Haselgrove, Steven~M Hodge, Pallavi~S Rane, Nikos Makris, and Jean~A Frazier.
\newblock Candishare: a resource for pediatric neuroimaging data.
\newblock {\em Neuroinformatics}, 10:319--322, 2012.

\bibitem{lpba}
David~W Shattuck, Mubeena Mirza, Vitria Adisetiyo, Cornelius Hojatkashani, Georges Salamon, Katherine~L Narr, Russell~A Poldrack, Robert~M Bilder, and Arthur~W Toga.
\newblock Construction of a 3d probabilistic atlas of human cortical structures.
\newblock {\em Neuroimage}, 39(3):1064--1080, 2008.

\bibitem{corrmlp}
Mingyuan Meng, Dagan Feng, Lei Bi, and Jinman Kim.
\newblock Correlation-aware coarse-to-fine mlps for deformable medical image registration.
\newblock In {\em Proceedings of the IEEE/CVF Conference on Computer Vision and Pattern Recognition}, pages 9645--9654, 2024.

\bibitem{transmatch}
Zeyuan Chen, Yuanjie Zheng, and James~C Gee.
\newblock Transmatch: a transformer-based multilevel dual-stream feature matching network for unsupervised deformable image registration.
\newblock {\em IEEE transactions on medical imaging}, 43(1):15--27, 2023.

\bibitem{dipy}
Eleftherios Garyfallidis, Matthew Brett, Bagrat Amirbekian, Ariel Rokem, Stefan Van Der~Walt, Maxime Descoteaux, Ian Nimmo-Smith, and Dipy Contributors.
\newblock Dipy, a library for the analysis of diffusion mri data.
\newblock {\em Frontiers in neuroinformatics}, 8:8, 2014.

\bibitem{dirlab}
Richard Castillo, Edward Castillo, Rudy Guerra, Valen~E Johnson, Travis McPhail, Amit~K Garg, and Thomas Guerrero.
\newblock A framework for evaluation of deformable image registration spatial accuracy using large landmark point sets.
\newblock {\em Physics in Medicine \& Biology}, 54(7):1849, 2009.

\bibitem{freesurfer}
Bruce Fischl.
\newblock Freesurfer.
\newblock {\em Neuroimage}, 62(2):774--781, 2012.

\bibitem{silu}
Dan Hendrycks and Kevin Gimpel.
\newblock Gaussian error linear units (gelus).
\newblock {\em arXiv preprint arXiv:1606.08415}, 2016.

\end{thebibliography}

\section{Appendix}

\subsection{Proofs}
\label{app:proof}
In order to prove Proposition \ref{prop}, we present two lemmas and then proceed to prove the proposition.

\begin{lemma}
\label{lem}
    The model $\phi(x, t)$ satisfying the composition rule of Eq. \ref{eq:composition} and $\phi(x, 0) = x$ is invertible for any $t \in (-1, 1)$:
\end{lemma}
\begin{proof}
    For an arbitrary $s \in (-1, 1)$, we can expand $\phi_s$ using the composition rule of Eq. \ref{eq:composition} to obtain
    \begin{gather}
    \label{eq:binary}
        \begin{split}
            \phi_s &= \phi_{\frac{s+1}{2}}\circ\phi_{\frac{s-1}{2}} \\
            &= (\phi_{\frac{s+3}{4}}\circ\phi_{\frac{s-1}{4}})\circ(\phi_{\frac{s+1}{4}}\circ\phi_{\frac{s-3}{4}}) \\
            &\vdots \\
            &= (\phi_{\frac{s}{2^n} + \frac{2^n - 1}{2^n}}\circ\phi_{\frac{s}{2^n} - \frac{1}{2^n}})\circ \dots \circ(\phi_{\frac{s}{2^n} + \frac{1}{2^n}}\circ\phi_{\frac{s}{2^n} - \frac{2^n - 1}{2^n}}).
        \end{split}
    \end{gather}

Eq. \ref{eq:binary} shows a binary tree expansion for $\phi_s$. At the $n^{\text{th}}$ level of the tree, there are $2^n$ leaf nodes. Furthermore, setting $s=0$ on both sides, and knowing that $\phi_0$ is the identity map, we will get

\begin{gather}
\label{eq:special}
    \phi_0 = Id = (\phi_{\frac{2^n - 1}{2^n}}\circ\phi_{\frac{-1}{2^n}})\circ \dots \circ(\phi_{\frac{1}{2^n}}\circ\phi_{\frac{1 - 2^n}{2^n}}),
\end{gather}
which implies that all the constituent factors must be invertible. Thus the composition rule of Eq. \ref{eq:composition} gives an arbitrarily fine grid $A_n \equiv \{\pm \frac{2m+ 1}{2^n}: 2m + 1 < 2^n, m \in \mathbb{Z}\}$ for any $n\in \mathbb{N}$ over the interval $(-1, 1)$, where $\phi$ is invertible on each grid point. For any $s\in(-1, 1)$ we can find an arbitrarily small neighborhood around a grid point $p_s \in A_n$ containing $s$. In the limit, $p_s$ coincides with $s$, and we will get
\begin{gather}
    \lim_{p_s\in A_n, n \rightarrow \infty} \phi_{p_s} = \phi_s,
\end{gather}
making $\phi_s$ invertible. We can conclude that a continuous $\phi_s$ is invertible for any $s\in(-1, 1)$. Moreover, due to the continuity of $\phi$, $\phi_1$ and $\phi_{-1}$ are invertible.
\end{proof}

\begin{corollary}
\label{cor}
    $\phi$ is closed under the composition operation.
\end{corollary}
\begin{proof}
    According to Lemma \ref{lem}, $\phi_s$ is invertible for any $s\in(-1, 1)$. For an arbitrary $s$, let $\psi_s$ denote the inverse of $\phi_s$. For any $s, t \in (-1, 1)$, both $\phi_s$ and $\phi_t$ are invertible. So, the composition $y = \phi_s\circ\phi_t(x)$ must also be invertible with an inverse mapping $x = \psi_{f(s, t)}(y)$, where $f: (-1, 1)^2 \rightarrow \mathbb{R}$. Therefore, we must have $\phi_s\circ\phi_t = \phi_{f(s, t)}$, indicating that $\phi$ is closed under composition.
\end{proof}

\begin{remark}
    There are $2^{\frac{n(n+1)}{2}}$ different ways to expand the binary tree in Eq. \ref{eq:special}. To see this, note that for one particular arrangement of factors at level $n$, we have $2^{n + 1}$ different configurations at level $n+1$. This is because two children nodes can be arranged in two possible ways for every parent node. Further, for $L_n$ different arrangements at level $n$, we have $L_{n+1} = 2^{n+1}L_n$ arrangements at level $n+1$. Noting that $L_1 = 2$, and working with the recurrence relation, we can find that $L_n = 2^{\frac{n(n+1)}{2}}$. This is merely an observation and does not affect any of the results.
\end{remark}

\begin{lemma}
    \label{lem2}
    For any two distinct grid points $p, q \in A_n, p\neq q$, the factors $\phi_p$ and $\phi_q$ appear next to each other in at least one of the $2^{\frac{n(n+1)}{2}}$ possible expansions in Eq. \ref{eq:special}.
\end{lemma}
\begin{proof}
    Due to the associative property of function composition, we can remove the parentheses in Eq. \ref{eq:special}. This gives us $2^n - 1$ compositions connecting $2^n$ factors. Remember that there are a total of $2^{\frac{n(n+1)}{2}}$ possible expansions in Eq. \ref{eq:special} at level $n$. Also, note that all the compositions appearing in all possible $n^{\text{th}}$ level expansions are well-defined since $\phi$ is closed under the composition operation as per Corollary \ref{cor}.

    The proposition is trivially true for $n=1$, because $\phi_0 = \phi_{1/2}\circ\phi_{-1/2} = \phi_{-1/2}\circ\phi_{1/2}$. We can also verify the proposition for $n=2$ with direct examination. For $n=2$ there are 4 constituent factors $\phi_{3/4}, \phi_{-1/4}, \phi_{-3/4},$ and $\phi_{1/4}$, and all $2(C_2^4) = 12$ compositions $\phi_{3/4}\circ\phi_{-3/4}, \phi_{3/4}\circ\phi_{-1/4}, \phi_{3/4}\circ\phi_{1/4}, \dots$, appear in $2^{\frac{2 \times 3}{2}} = 8$ possible expansions.

    Assume that the proposition holds for $n-1$. Take $p, q \in A_n$ and $p\neq q$. If $p$ and $q$ belong to the same parent in $A_{n-1}$, then $\phi_p$ and $\phi_q$ appear next to each other. On the other hand, if they belong to two different parents in $A_{n-1}$, then, based on our assumption, these two parents must have appeared next to each other at least once. In this case, there are 4 children factors including $\phi_p$ and $\phi_q$, and we can verify that any two of them appear next to each other using a direct verification similar to what we did for the case $n=2$. Hence, by mathematical induction, the proposition is true for any $n \in \mathbb{N}$.
\end{proof}

\textbf{Proposition 1.} \textit{A differentiable deformation $\phi(x, t)$ that satisfies the composition rule of Eq. \ref{eq:composition} and $\phi(x, 0) = x$ is an exponential map, or equivalently, it is a one-parameter diffeomorphism solving an autonomous ODE of the form of Eq. \ref{eq:ode}.}

\begin{proof}
    We turn our attention to the function $f$ used to define the composition $\phi_s\circ \phi_t = \phi_{f(s, t)}$. Thus, applying the notion of the function $f$ to Eq. \ref{eq:special}, we must have

    \begin{gather}
    \label{eq:f}
        0 = f(\frac{2^n - 1}{2^n}, f(-\frac{1}{2^n}, \dots, f(\frac{1}{2^n}, -\frac{2^n - 1}{2^n}))).
    \end{gather}
    
Defining $f$ to be the addition operation ($+$), Eq. \ref{eq:f} becomes consistent. Moreover, note that the order of composition can start with any two factors appearing next to each other in Eq. \ref{eq:special} (due to the associative property). According to the Lemma \ref{lem2}, for any $p, q\in A_n$ and $p\neq q$, we know that $\phi_p$ and $\phi_q$ appear next to each other in at least one of the $2^{\frac{n(n + 1)}{2}}$ possible expansions in Eq. \ref{eq:special}. Consistent evaluations of these $2^{\frac{n(n + 1)}{2}}$ expressions allow us to define $f(p, q) = p + q$ for any $p, q \in A_n$, for all $n$, such that $p\neq q$. Therefore, we will have $\phi_p \circ \phi_q = \phi_{p + q}$ for any $p, q\in A_n$ and $p\neq q$. Letting $p_s$ and $p_t$ denote the closest grid points to $s, t\in (-1, 1)$, where $s\neq t$, respectively, and taking the limits, we will get

\begin{gather}
    \begin{split}
        \phi_s\circ \phi_t &= \lim_{p_s, p_t\in A_n, n\rightarrow \infty} \phi_{p_s}\circ \phi_{p_t} \\
        &= \lim_{p_s, p_t\in A_n, n\rightarrow \infty} \phi_{p_s + p_t} \\
        &= \phi_{s + t}.
    \end{split}
\end{gather}

Additionally, taking the limit, we will get,

\begin{gather}
    \phi_s \circ \phi_s = \lim_{t\rightarrow s, t\in(-1, 1)} \phi_s\circ\phi_t = \lim_{t\rightarrow s, t\in(-1, 1)} \phi_{s + t} = \phi_{2s}.
\end{gather}

This proves that $\phi_s$ is a one-parameter family of transformations. Additionally, $\phi$ is differentiable, and therefore, it is an exponential map \cite{biagi2018introduction}.
\end{proof}

\subsection{Datasets Preparation}
\label{app:data}

\paragraph{OASIS Dataset}
The Open Access Series of Imaging Studies (OASIS) dataset\footnote{\url{hhtps://github.com/adalca/medical-datasets/blob/master/neurite-oasis.md}} \cite{oasis} contains 416 T1 weighted scans from subjects aging from 18 to 96 with 100 of them diagnosed with mild to moderate Alzheimer's disease. The segmentation masks of 35 subcortical regions available in the dataset serve as the ground truth for further evaluation of the registration. By random sampling we generate 1000 pairs for training, 100 pairs for validation and 1000 pairs for test.

\paragraph{CANDI Dataset}
The Child and Adolescent NeuroDevelopment Initiative (CANDI) dataset\footnote{\url{https://www.nitrc.org/projects/candi_share/}} \cite{candi} contains T1 weighted brain scans of 117 subjects divided into 4 different subgroups including Healthy Control (\textbf{HC}), Schizophrenia Spectrum (\textbf{SS}), Bipolar Disorder with Psychosis (\textbf{BPDwithPsy}), and Bipolar Disorder without Psychosis (\textbf{BPDwithoutPsy}). The segmentation masks of 32 subcortical regions available in the dataset are used as the ground truth for the evaluation. For training and evaluation, we mix all the groups, and from 117 subjects, we use 80, 11, 26 subjects for training, validation, and testing, respectively. We generate 400 pairs for training, 25 pairs for validation, and 300 pairs for test.

\paragraph{LPBA40 Dataset}
The LONI Probabilistic Brain Atlas (LPBA40) dataset\footnote{\url{https://www.loni.usc.edu/research/atlas_downloads}} \cite{lpba} comes with 40 scans and the segmentation maps of 56 cortical and subcortical regions. For the LPBA40 dataset 190 pairs from 20 subjects are used as the training set, 10 pairs from 5 subjects are used as the validation set and 105 pairs from 15 subjects are used for the test set.

\paragraph{IXI Dataset}
The Information eXtraction from Images (IXI) dataset \footnote{\url{https://brain-development.org/ixi-dataset}} contains 581 brain scans of normal, healthy subjects. From all 581 scans, we use 450, 50, 81 subjects for training, validation, and testing, respectively, from which we create 1500, 100, 300 random pairs for training, validation, and testing.

\paragraph{DIRLAB Dataset}
\label{par:dirlab}
The Deformable Image Registration LABoratory (DIRLAB) dataset \cite{dirlab} contains 10 pairs of lung CT scans corresponding to 10 patients. Each pair consists of the scans of lungs in the extreme inhaling and exhaling phases. The dataset also comes with 300 landmark points per image, annotated manually by experts, which are used for the evaluation. For the DIRLAB dataset, 8 subjects are used for training and 2 subjects are used for test.

\paragraph{Data Preparation}
\label{par:dp}
For the brain scan dataset we use the skull stripped, MNI152 1mm normalized images and center-cropped the images to the size $160\times144\times192$. For all datasets we use all the available anatomical structures for evaluation.  We use min-max intensity normalization to transform the voxel intensities into the $[0, 1]$ range. For the IXI dataset the segmentation maps are not readily available. Therefore, we used the FreeSurfer software \cite{freesurfer} to produce the segmentation maps with 32 subcortical anatomical structures. For the DIRLAB dataset, the images are resampled to isotropic voxels of size $1mm$ and cropped to the same size of $96\times256\times256$.

\subsection{Definition of the NCC loss function}
\label{app:ncc}

The NCC loss function used in this paper is similar to the definition used by previous methods \cite{voxelmorph, symnet, nodeo, dnvf}. For a pair of images $I$ and $J$ the localized NCC is defined as

\begin{gather}
    \begin{split}
        Nom(I, J, W_p) &= \left(\sum_{p_i\in W_p}(I(p_i) - \mu_{I_{W_p}}(p_i))(J(p_i) - \mu_{J_{W_p}}(p_i))\right)^2\\
        Denom(I, J, W_p) &= \sum_{p_i\in W_p}(I(p_i) - \mu_{I_{W_p}}(p_i))^2 . \sum_{p_i\in W_p}(J(p_i) - \mu_{J_{W_p}}(p_i))^2\\
        NCC(I, J) &= \sum_{W_p, p\in\Omega}\frac{Nom(I, J, W_p)}{Denom(I, J, W_p)},
    \end{split}
\end{gather}

where $\Omega$ is the definition domain of images and $W_p$ is a local window around the point $p\in\Omega$. The terms $I(p_i)$ and $J(p_i)$ denote the value of image $I$ and $J$ at the point $p_i$, respectively. Also, $\mu_{I_{W_p}}$ and $\mu_{J_{W_p}}$ denote the mean intensity of $I$ and $J$ in the local window $W_p$, respectively. The highest value for NCC is 1 in the case that $I$ and $J$ perfectly match, and the lowest value is 0 for the case of no correspondence.

\subsection{Evaluation Metrics}
\label{app:metrics}

\paragraph{Dice Score}
Following the convention of evaluating diffemorphic image registration methods (\cite{dnvf, nodeo, symnet, lapirn}), we use \textbf{Dice Similarity Coefficient} to measure the similarity of the registered and target segmentation maps for the brain scan datasets. Dice similarity coefficient is a method to measure the amount of overlap between the registered image and the target image. For our experiments we use the segmentation masks of the deformed image and the fixed image to measure the Dice coefficient. We follow the convention of VoxelMorph \cite{voxelmorph} for computing the Dice score, which measures the mean Dice over all anatomical structures.

\paragraph{HD95}
This metric computes 95\% percentile of the Hausdorff distance of the segmentation maps.

\paragraph{Percentage of Negative Jacobian Determinant}
The Jacobian determinant of the deformation (denoted by $|det(J_\phi)<0|$) at each point determines if the deformation is a local diffeomorphism (according to the Inverse Function Theorem) \cite{symnet}. Having positive Jacobian determinant at any point means an orientation preserving diffeomorphism at that point. Therefore, the percentage of the voxels at which the deformation exhibits non-positive Jacobian determinant is a measure of how well the model is preserving the topology. A lower the percentage of non-positive Jacobian determinant is an indication of better preservation of the topology and avoiding folding in the grid.

\paragraph{Target to Registration Errors}
This metric measures the error of registration in terms of the average distance between the landmark points of the warped image and the target fixed image. This metric is used for the DIRLAB dataset for which the landmark points are available. Since the images are resampled into $1mm$ isotropic voxels the TREs are reported in millimeter.

\subsection{Comparison Models}
\label{app:models}

For evaluating the performance of SGDIR, we have included the performance of 10 diffeomorphic methods and 5 non-diffeomorphic methods. Table \ref{tab:methods} summarizes these methods. For the VoxelMorph, we have used the diffeomorphic version of the model.

\begin{table}[!ht]
  \caption{Categorization of the methods used for comparisons.}
  \label{tab:methods}
  \centering
  \resizebox{0.4\textwidth}{!}{
  \begin{tabular}{lc}
    \toprule
    Model & Diffeomorphic \\
    \midrule
    SyN \cite{syn} & $\checkmark$ \\
    VoxelMorph \cite{voxelmorph} & $\checkmark$ \\
    SYMNet \cite{symnet} & $\checkmark$ \\
    NePhi \cite{nephi} & $\checkmark$ \\
    IDIR \cite{idir} & $\checkmark$ \\
    NODEO \cite{nodeo} & $\checkmark$ \\
    DNVF \cite{dnvf} & $\checkmark$ \\
    TransMorph-diff \cite{transmorph} & $\checkmark$ \\
    PULPo \cite{pulpo} & $\checkmark$ \\
    LapIRN \cite{lapirn} & $\checkmark$ \\
    \midrule
    DiffuseMorph \cite{diffusemorph} & $\times$ \\
    H-ViT \cite{hvit} & $\times$ \\
    CorrMLP \cite{corrmlp} & $\times$ \\
    CycleMorph \cite{cyclemorph} & $\times$ \\
    TransMatch \cite{transmatch} & $\times$ \\
    \bottomrule
  \end{tabular}}
\end{table}

\subsection{Implementation Details}
\label{app:impl}

SGDIR employs a straightforward time-embedded UNet architecture as shown in Figure \ref{fig:schematic}, The UNet encoder has 4 down-sampling layers of dimensions 32, 64, 128, 128, and 256, and the decoder has 4 up-sampling layers with the same dimensions as the down-sampling layers but in the reversed order. The time-embedding module consists of a sinusoidal positional encoding where the dimension is set to 64. All the activation functions for the layers are set to SiLU \cite{silu} to provide more smoothness to the network. In our experiments the regularization factor for the semigroup term is set to $\lambda = 5e^5$, while an ablation study on the choice of $\lambda$ is provided at \ref{sec:lambda}. The NCC window size is set to 11 chosen empirically from a set of values $\{3, 5, 7, 9, 11, 13, 17\}$. The training is carried out for 200 epochs using Adam optimizer with the learning rate $lr = 1e^{-4}$. The whole implementation is done in PyTorch and is tested on both NVIDIA GTX 1080 Ti and NVIDIA Tesla V100 GPUs with 32GB RAM. Table \ref{tab:comp} reports the computational cost of SGDIR over MR scans of size $160\times144\times192$, and compares it to the computational requirements of H-ViT, as a powerful non-diffeomorphic method. In this table, \textit{Train Itr. Time} denotes the time it takes for the model to train over a single pair of images, and similarly the inference time is reported based on a full registration process for a given pair of images at inference time. This table also reveals that the inference from SGDIR is relatively cheap in terms of memory usage and inference time, making it suitable for practical scenarios. This is mainly because that during the inference, we do not need the extra forward calls that were necessary during the training.

\begin{table}[!ht]
  \caption{Parameters/Memory Requirements of SGDIR vs H-ViT}
  \label{tab:comp}
  \centering
  \resizebox{0.5\textwidth}{!}{
  \begin{tabular}{lll}
    \toprule
    Variable & SGDIR & H-ViT \\
    \midrule
    Params. (\#M) & 16.21 & 21.23 \\
    Max. Mem. (Train + Val) (GB) & \textbf{21.53} & 22.60 \\
    Max. Mem. (Inference) (GB) & \textbf{0.17} & 0.38\\
    Train Itr. Time (s) & \textbf{16.8} &  22.06 \\
    Inference Time (s) & \textbf{0.23} & 0.31 \\
    \bottomrule
  \end{tabular}}
\end{table}

\subsection{Remark on the Semigroup Property}
\label{app:discussion}

Here we present a brief discussion on the importance of our particular design of semigroup property. The original semigroup property of Eq. \ref{eq:sg} must stand for arbitrary $t$ and $s$. If a model only satisfies Eq. \ref{eq:sg}, it does not yield a unique deformation. The data fidelity -i.e., how the deformation acts on the images- is necessary to find the correct deformation. If one intends to directly force Eq. \ref{eq:sg} for arbitrary $t$ and $s$, the data fidelity term required to find the correct deformation would be to match $\phi_t[I_m]$ to $\phi_{t-1}[I_f]$, and $\phi_s[I_m]$ to $\phi_{s-1}[I_f]$. Comparing to our formulation (Eq. \ref{eq:loss_sim} and Eq. \ref{eq:reg}), this pushes us to have 2 more forward calls for computing $\phi_s[I_m]$ and $\phi_{s-1}[I_f]$, which makes the training harder and more computationally demanding. However, we have shown that only the first data fidelity term, and its compositional rule suffices to ensure that we are actually learning the solution to our deformation ODE, i.e., a diffeomorphic deformation, without having more computational overhead.

\subsection{Experiment on Lung CT Dataset}
\label{app:dirlab}

To show that SGDIR is not limited to 3D brain MR scans, we have conducted a small experiment on the \textbf{DIRLAB} dataset \cite{dirlab} which contains the lung scans of 10 subjects in the extreme inhaling and exhaling phases. The detail of the dataset and our data preparation is explained in Sec. \ref{par:dirlab}. Table \ref{tab:dice4} demonstrates the advantage of SGDIR in terms of the TRE metric over the DIRLAB dataset. This table compares the models based on the TRE metric. Since the images are resampled to isotropic voxel size of $1mm$, the error is presented in $mm$ precision.

\begin{table}[!ht]
  \caption{Performance comparison of SGDIR with other methods on the DIRLAB dataset.}
  \label{tab:dice4}
  \centering
  \resizebox{0.7\textwidth}{!}{
  \begin{tabular}{lccc}
    Category & Model & \small{TRE (mm)} ($\downarrow$) & $|\text{det}(J_\phi) < 0|$ ($\downarrow$)\\
    \midrule
    Traditional & SyN & 1.35 & 0.0034\% \\
    \midrule
    \multirow{3}{*}{Learning} & VoxelMorph & 1.33 &  0.0011\% \\
     & SYMNet & 1.27 & 0.0006\% \\
     & LapIRN & 1.21 & \textbf{0.0001}\%\\
     & DiffuseMorph & 1.35 & 0.0010\% \\
    \midrule
    \multirow{2}{*}{Optimization} & IDIR & 1.13 & 0.0015\% \\
     & NODEO & 1.08 & 0.0008\% \\
    \midrule
    Our Method & SGDIR & \textbf{0.98} & 0.0003\% \\
    \bottomrule
  \end{tabular}}
\end{table}

Quantitative results reveal that SGDIR is capable of producing accurate deformation with less than $1mm$ average error, demonstrating competitive performance with respect to instance optimization-based methods. Also, Figure \ref{fig:comparisons_dir} illustrates the visual comparison between SGDIR, IDIR, NODEO, and VoxelMorph over the DIRLAB dataset. The visual results show that the SGDIR demonstrates a highly precise and smooth deformation field over the landmark points inside the lungs.

Even though the experiment is conducted over a rather small dataset, it corroborates the utility of SGDIR in providing smooth and accurate registrations for other modalities other than MRI.

\subsection{More Visual Results}
\label{app:visuals}

In this section we provide more qualitative results illustrating the visual behavior of SGDIR and the comparisons of SGDIR with respect to other methods. Figure \ref{fig:timsteps} illustrates the evaluation of SGDIR for various number of sampled time steps for the subject ids 1 and 2 of the OASIS dataset. In this figure, both forward and backward deformations along with their warped images are illustrated. The results reveal that a continuous sampling of the time steps is highly in favor of the performance of the model both in terms of Dice score and folding percentage. 

\begin{figure*}[!ht]
    \centering
    \includegraphics[width=0.7\textwidth]{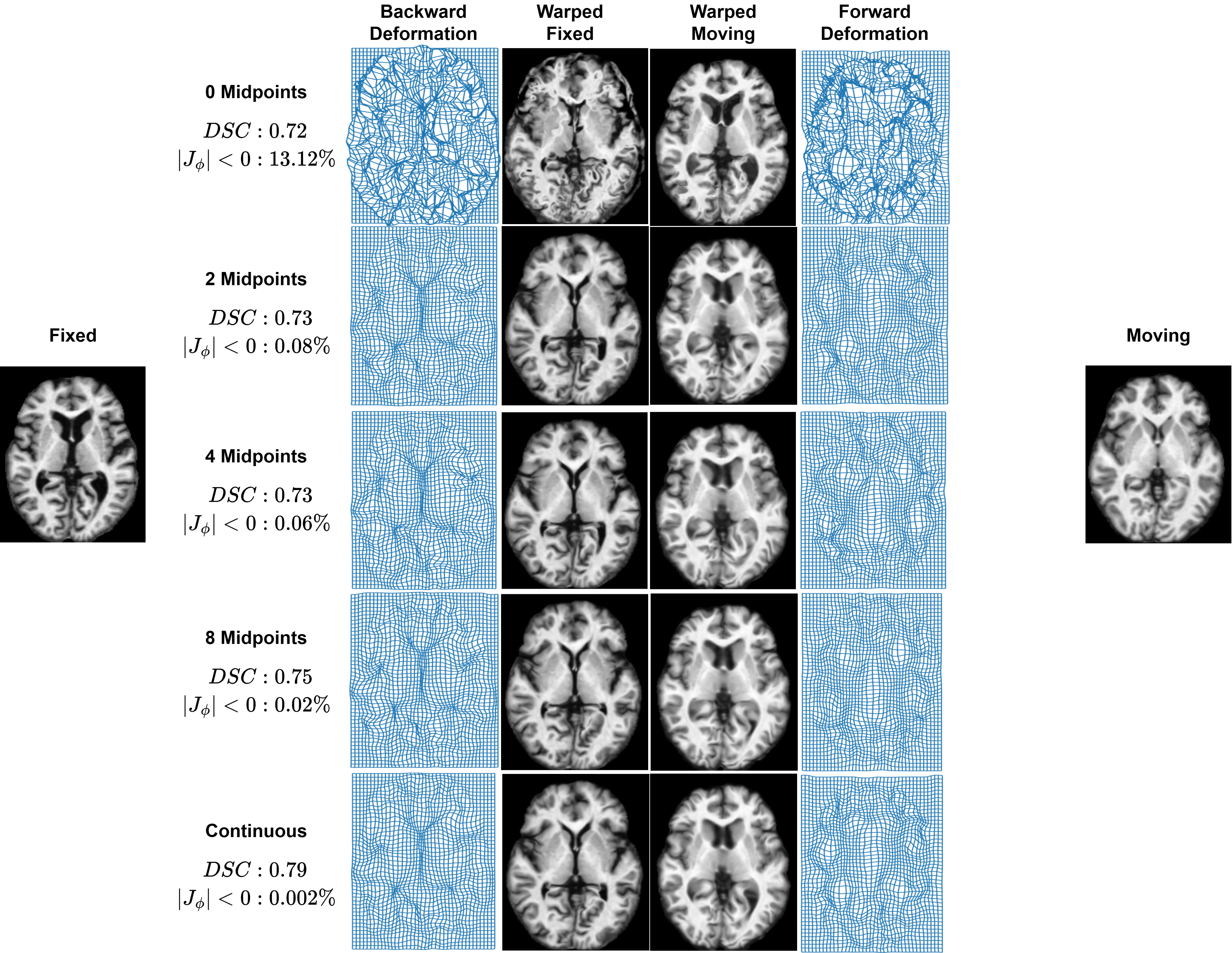}
    \caption{The effect of number of intermediate time steps on the smoothness (percentage of negative Jacobian determinants) and the Dice score. }
    \label{fig:timsteps}
\end{figure*}

Figures \ref{fig:comparisons2} and \ref{fig:comparisons3} depict visual comparisons of SGDIR with other methods over the CANDI and LPBA40 datasets, respectively. Also, Figure \ref{fig:comparisons_ixi} illustrates the qualitative comparisons of SGDIR with other methods over the IXI dataset. 

\begin{figure}[!ht]
    \centering
    \includegraphics[width=\textwidth]{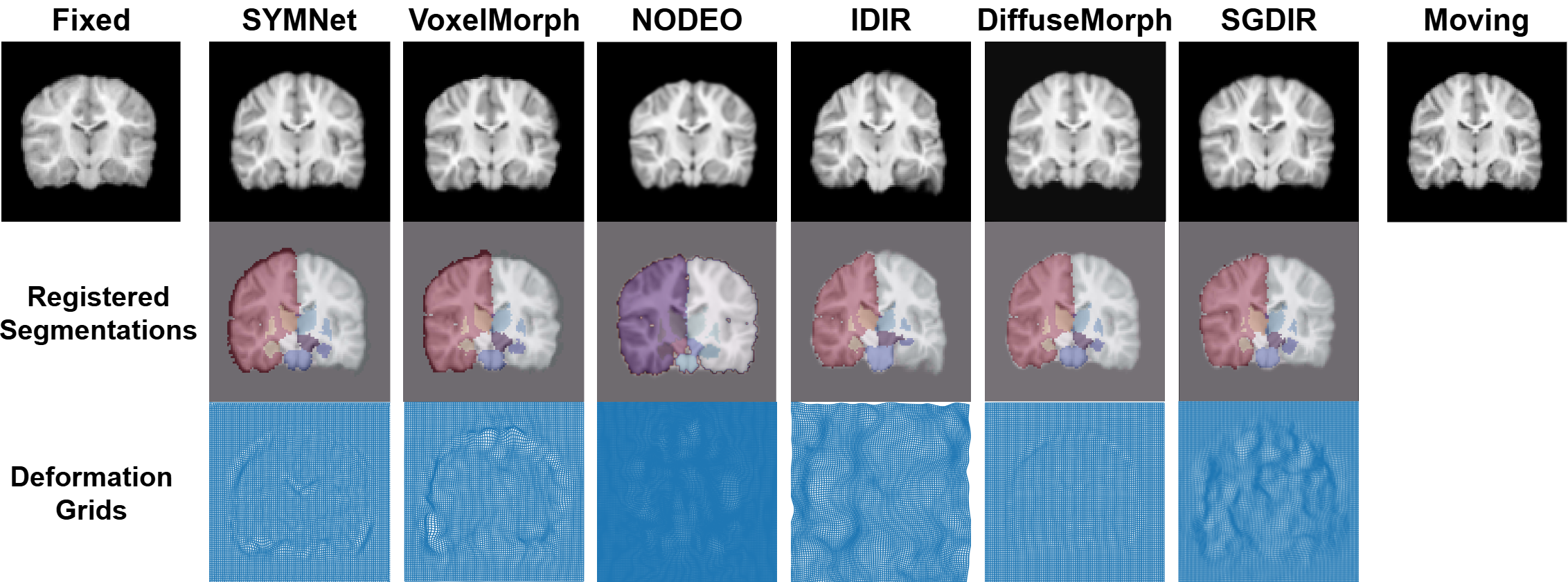}
    \caption{Qualitative comparison of SGDIR with learning-based and optimization-based methods including SYMNet, VoxelMorph, NODEO, IDIR, and DiffuseMorph over the CANDI dataset.}
    \label{fig:comparisons2}
\end{figure}

\begin{figure}[!ht]
    \centering
    \includegraphics[width=0.75\textwidth]{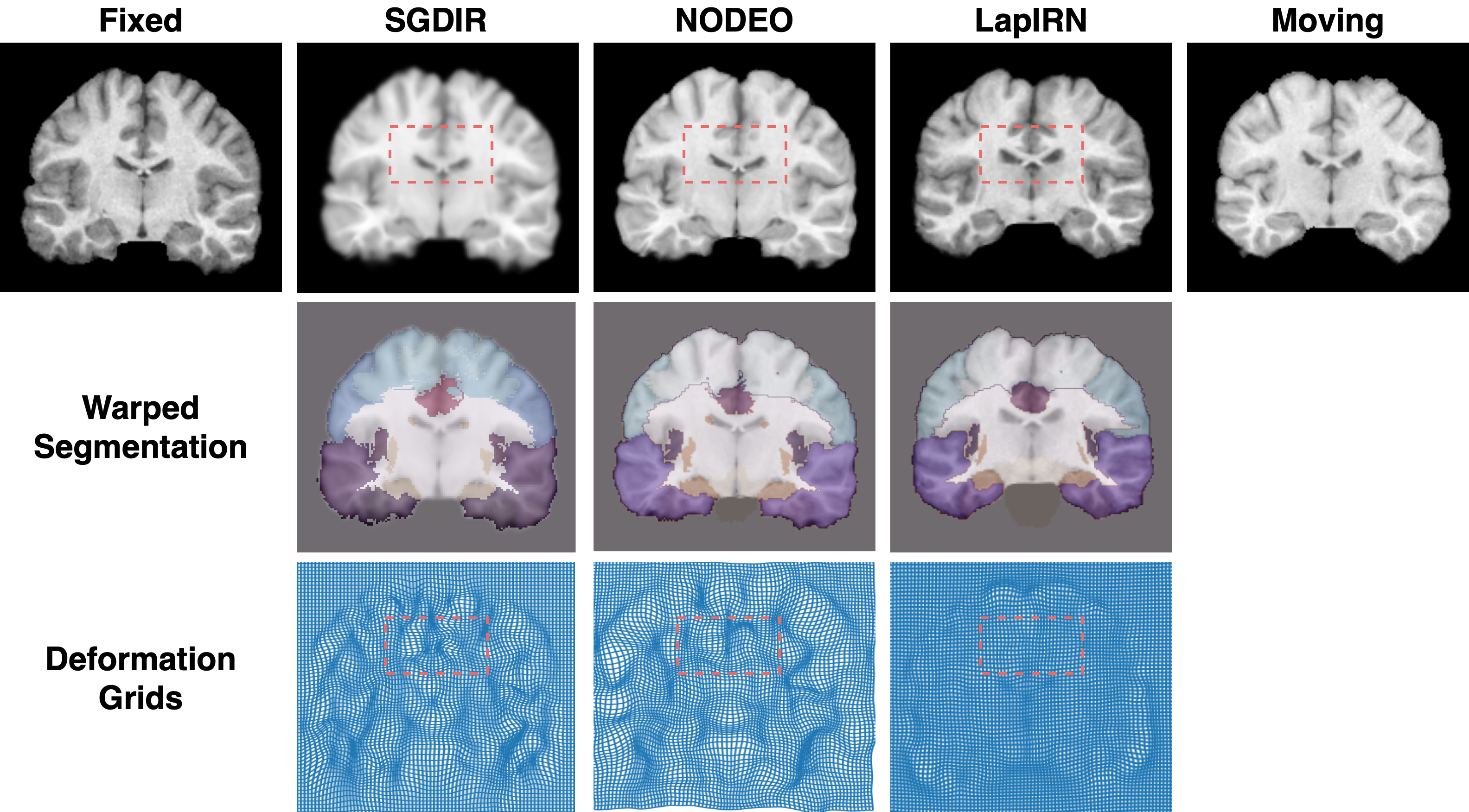}
    \caption{Qualitative comparison of SGDIR with learning-based and optimization-based methods including NODEO, LapIRN over the LPBA40 dataset.}
    \label{fig:comparisons3}
\end{figure}

To better showcase the advantage of SGDIR with respect to non-diffeomorphic methods, we have provided a visual comparison of SGDIR with H-ViT as a successful recent non-diffeomorphic deformable method. Figure \ref{fig:hvitsgdir} illustrates the comparison between the behavior of SGDIR as a diffeomorphic method with H-ViT. This visual comparison shows how SGDIR can output significantly smoother deformation grids (with orders of magnitude better lower percentage of negative Jacobian determinant percentage) while keeping visual quality, which is a crucial requirement in medical image analysis. 

\begin{figure}[!ht]
    \centering
    \includegraphics[width=0.6\textwidth]{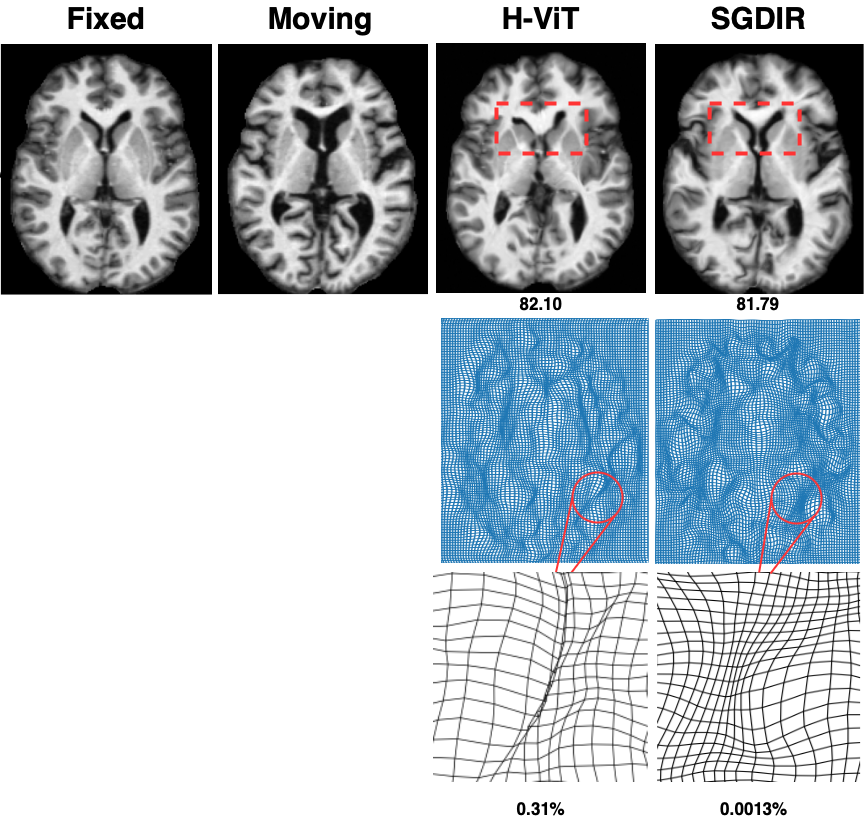}
    \caption{Qualitative comparison of H-ViT and SGDIR on a pair of samples in OASIS dataset. First row shows the warped images under each model. Below the warped images the corresponding dice scores are presented. The second and third rows show the deformation grids and their zoomed in versions, respectively. Below the third row, the percentage of negative Jacobian determinants are presented. Even though H-ViT demonstrates higher dice score, it suffers from high folding percentage which is crucial for medical downstream tasks.}
    \label{fig:hvitsgdir}
\end{figure}

\begin{figure*}[!ht]
    \centering
    \includegraphics[width=\textwidth]{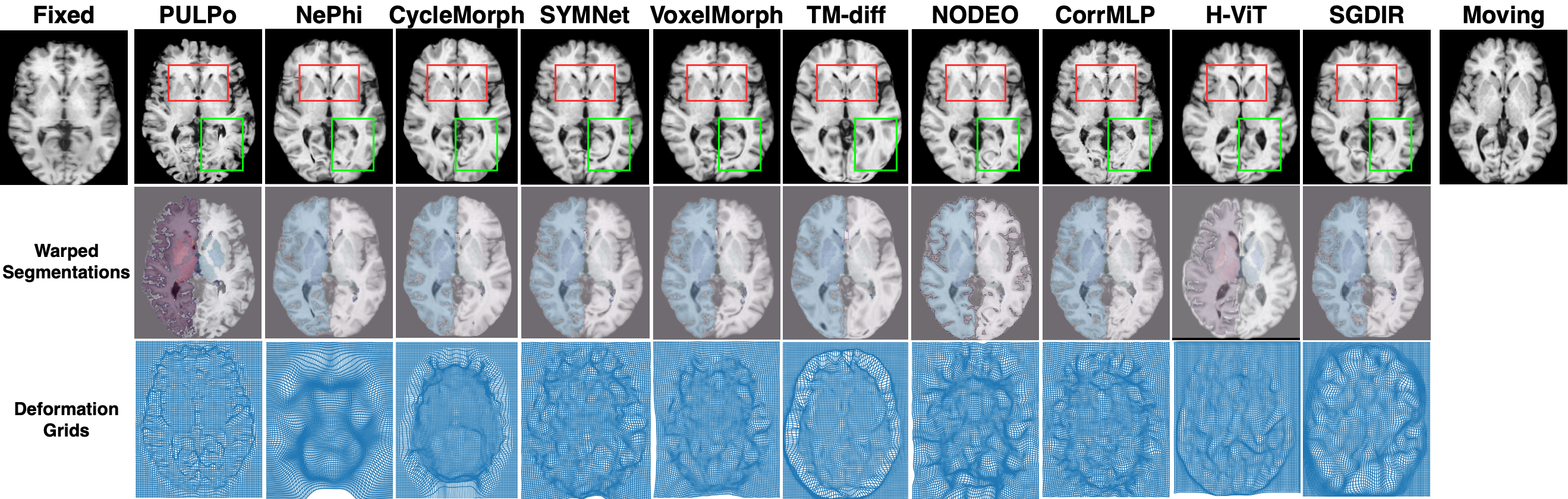}
    \caption{Qualitative comparison of SGDIR with other methods over the IXI dataset.}
    \label{fig:comparisons_ixi}
\end{figure*}

\begin{figure*}[!ht]
    \centering
    \includegraphics[width=0.8\textwidth]{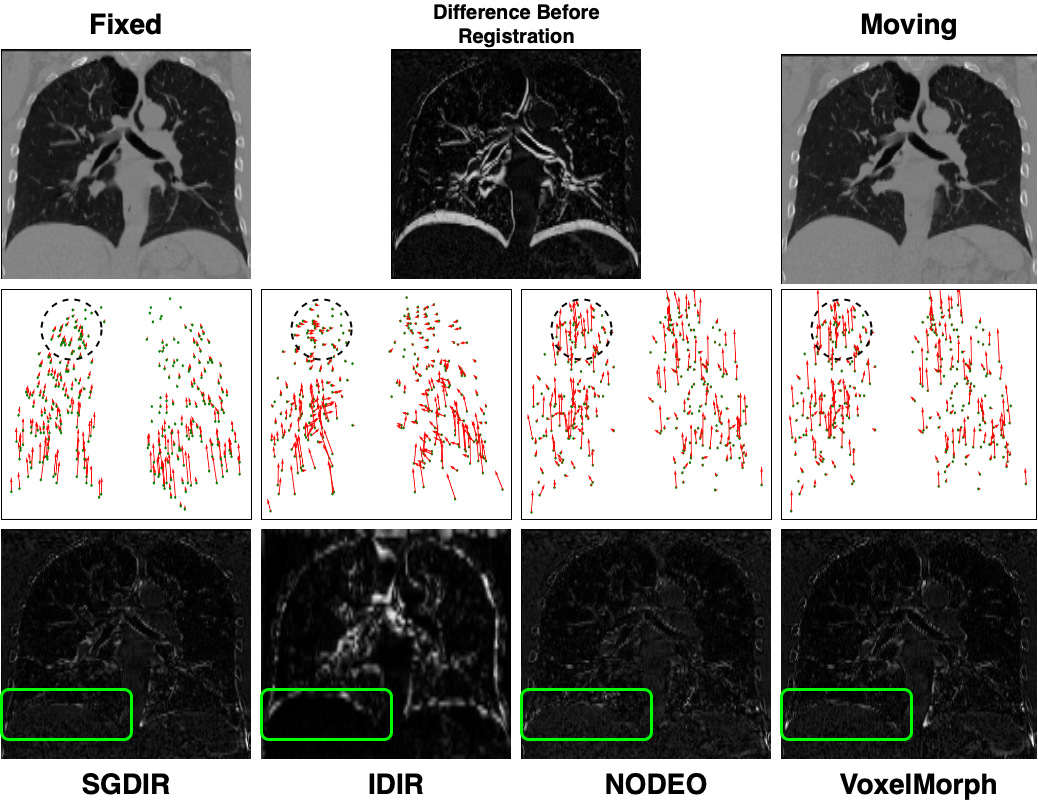}
    \caption{Qualitative comparison of SGDIR with other methods over the DIRLAB lung dataset. The first row depicts the fixed and moving images along with their absolute difference before registration. The second row illustrates the deformation fields computed over the landmark points by each model. The third row shows the difference of warped moving image and the fixed image after the registration by each model.}
    \label{fig:comparisons_dir}
\end{figure*}

Finally, Figure \ref{fig:boxplots} draws the boxplot distribution of dice scores in each anatomical region for the OASIS (top), CANDI (middle), and IXI (bottom) datasets.

\begin{figure*}[!t]
    \centering
    \includegraphics[width=\textwidth]{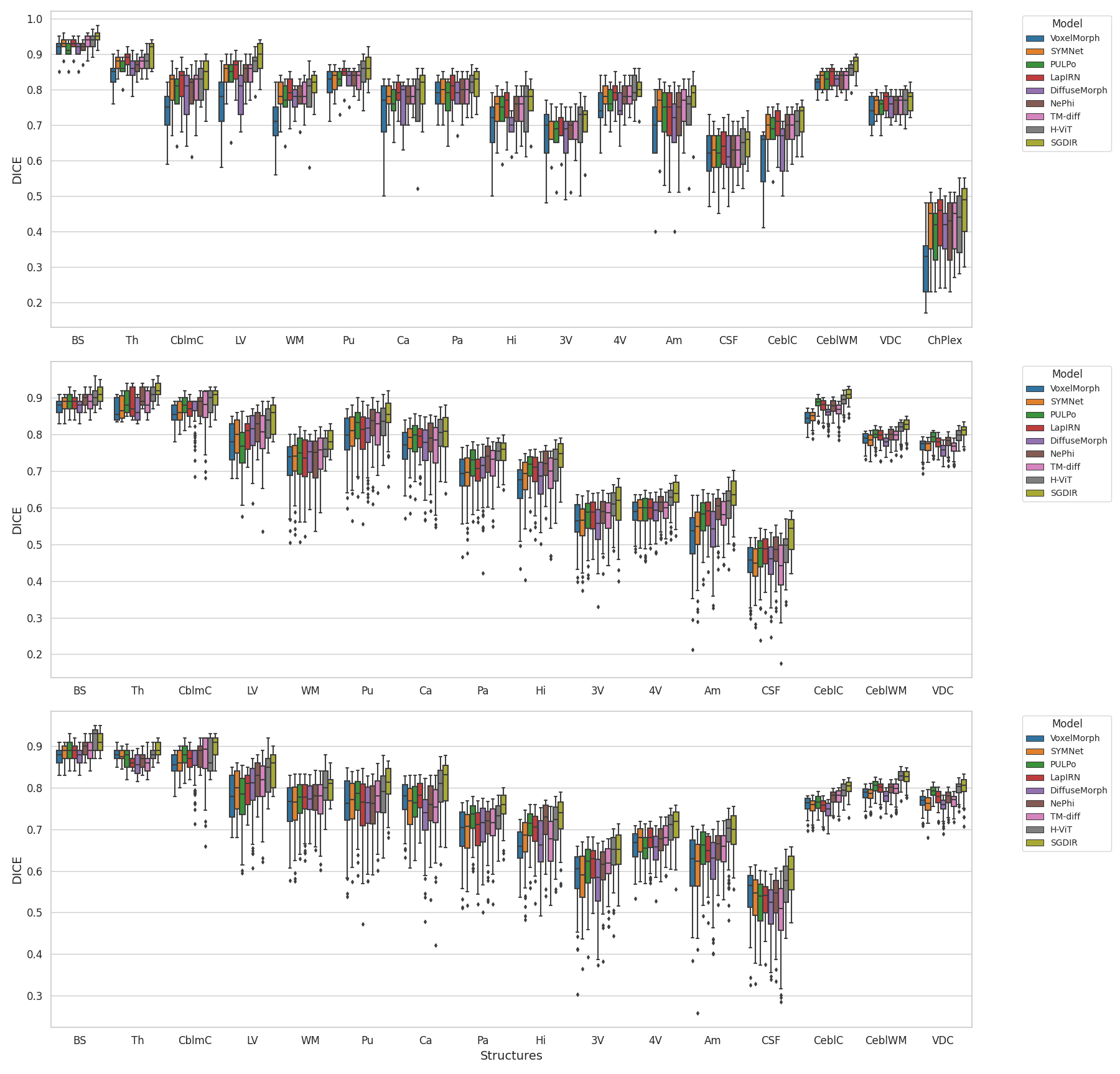}
    \caption{The boxplots of DICE score for each anatomical structure for the OASIS dataset (top), CANDI dataset (middle), and the IXI dataset (bottom). The structures are averaged for left and right hemisphere. The anatomical structures are as follows: Brain Stem (BS), Thalamus (Th), Cerebellum Cortex (ClbmC), Lateral Ventricle (LV), Cerebellum White Matter (WM), Putamen (Pu), Caudate (Ca), Pallidum (Pa), Hippocampus (Hi), 3rd Ventricle (3V), 4th Ventricle (4V), Amygdala (Am), CSF (CSF), Cerebral Cortext (CeblC), Cerebral White Matter (CeblWM), Ventral DC (VDC), Choroid Plexus (ChPlex).}
    \label{fig:boxplots}
\end{figure*}

\end{document}